\documentclass{article} % For LaTeX2e
\usepackage{iclr2025_conference,times}
\usepackage{amsmath,amsfonts,bm}

\def\eqref#1{equation~\ref{#1}}
\def\1{\bm{1}}

\DeclareMathAlphabet{\mathsfit}{\encodingdefault}{\sfdefault}{m}{sl}
\SetMathAlphabet{\mathsfit}{bold}{\encodingdefault}{\sfdefault}{bx}{n}

\definecolor{linkColor}{rgb}{0.18,0.39,0.62}
\usepackage[pagebackref=true,breaklinks=true,colorlinks,citecolor=linkColor]{hyperref}
\usepackage{url}
\newcommand{\myparagraph}[1]{\vspace{0.1em}\noindent\textbf{#1}}
\newcommand{\ie}{\textit{i}.\textit{e}.}
\newcommand{\eg}{\textit{e}.\textit{g}.}

\usepackage{pifont}
\usepackage{arydshln}
\newcommand{\cmark}{\ding{51}}
\newcommand{\xmark}{\ding{55}} 

\usepackage{bm,amsmath}
\usepackage{arydshln}
\usepackage{multirow}
\usepackage{graphicx}
\usepackage{color}
\usepackage{xcolor}
\usepackage{colortbl}
\definecolor{ours}{gray}{.95}
\usepackage{wrapfig}
\usepackage{amsthm}
\newtheorem{theorem}{Theorem}[section]
\newtheorem{lemma}[theorem]{Lemma}

\usepackage{graphicx}
\usepackage{pythonhighlight}
\title{Memory Efficient Transformer Adapter for Dense Predictions}
\author{Dong Zhang$^{1,2}$, Rui Yan$^{3}$, Pingcheng Dong$^{1}$, Kwang-Ting Cheng$^{1}$\\
$^1$The Hong Kong University of Science and Technology\\
$^2$InnoHK AI Chip Center for Smart Emerging Systems,~$^3$Nanjing University\\
\small \texttt{\{dongz,timcheng\}@ust.hk;ruiyan@nju.edu.cn;pingcheng.dong@connect.ust.hk}
}
\iclrfinalcopy % Uncomment for camera-ready version, but NOT for submission.
\begin{document}
% ----------------------------------------------
\maketitle
% ----------------------------------------------
\begin{abstract}
While current Vision Transformer (ViT) adapter methods have shown promising accuracy, their inference speed is implicitly hindered by inefficient memory access operations, \eg, standard normalization and frequent reshaping. In this work, we propose META, a simple and fast ViT adapter that can improve the model's memory efficiency and decrease memory time consumption by reducing the inefficient memory access operations. Our method features a memory-efficient adapter block that enables the common sharing of layer normalization between the self-attention and feed-forward network layers, thereby reducing the model's reliance on \emph{normalization operations}. Within the proposed block, the cross-shaped self-attention is employed to reduce the model's frequent \emph{reshaping operations}. Moreover, we augment the adapter block with a lightweight convolutional branch that can enhance local inductive biases, particularly beneficial for the dense prediction tasks, \eg, object detection, instance segmentation, and semantic segmentation. The adapter block is finally formulated in a cascaded manner to compute diverse head features, thereby enriching the variety of feature representations. Empirically, extensive evaluations on multiple representative datasets validate that META substantially enhances the predicted quality, while achieving a new state-of-the-art accuracy-efficiency trade-off. Theoretically, we demonstrate that META exhibits superior generalization capability and stronger adaptability.
\end{abstract}
% ---------------------------------------------------
% \keywords{Vision Transformer, Vision Adapter, Memory Efficient Learning, Dense Prediction Tasks, Segmentation}
% ----------------------------------------------
% -----------------------------------
\section{Introduction}
\label{sec:intro}
% -----------------------------------
State-of-the-art computer vision models typically follow the intuitive paradigm of pre-training on large single-modal general datasets, followed by fine-tuning on local task-specific datasets to achieve promising accuracy~\citep{chen2021pre,zhang2024boundary,radford2021learning,he2022masked,zhang2023cae}. 
% For example, in the semantic segmentation (SSeg) models, pre-trained backbone parameters based on ImageNet~\citep{deng2009imagenet} or MS-COCO~\citep{lin2014microsoft} are commonly used for downstream model initialization~\citep{he2016deep,xie2021segformer,liu2021swin,zhang2022deep}. 
However, this defacto paradigm requires downstream models to load the entire pre-trained model for fine-tuning, which may result in several undesirable drawbacks such as poor structural flexibility, large optimization gaps, and limited applicability in certain scenarios~\citep{jie2023fact,du2022survey,jie2022convolutional}. More importantly, these drawbacks have become more prominent and urgent problems that need to be addressed for dense prediction tasks, particularly in the current context of significant model size growth~\citep{gao2023clip,zhang2022unabridged,radford2021learning,zhang2023cae,Kirillov_2023_ICCV}.

Recently, with the increasing dominance of Vision Transformer (ViT) architectures~\citep{han2022survey,liu2021swin,khan2022transformers}, following the parameter-efficient transfer learning mechanism~\citep{houlsby2019parameter}, ViT adapter has become a central approach to learning vision-specific inductive biases from pre-trained ViT models~\citep{hu2022lora,jie2023fact,chen2022vision,ma2024segment,luo2023forgery,shao2023deepfake}, successfully addressing the drawbacks associated with the pre-training followed by fine-tuning paradigm.
The progressive ViT adapter enables downstream ViT models to achieve promising accuracy levels that are comparable to, or even higher than, those achieved by fine-tuning the entire model. 
% Fortunately, ViT adapter is usually accomplished using only approximately $5\%$ of the learning parameters~\citep{jie2023fact,chen2022vision,shao2023deepfake}. The classic ViT adapter is pre-training-free and mainly consists of a spatial \emph{feature extractor} and a spatial \emph{feature injector}~\citep{chen2022vision,hu2022lora,luo2023forgery,marouf2024mini}. The former extracts general features from a pre-trained ViT backbone network within each block, while the latter interacts the updated task-specific features with the ViT model~\citep{marouf2024mini,chen2022adaptformer,dong2024efficient}. In training, only the adapter's parameters need to be updated, while the ViT backbone's parameters remain fixed. 
In particular, thanks to the utilization of plain ViT models, which contain unique multi-modal information, adapters based on these models can effectively promote downstream models to learn beneficial semantic-rich feature representations~\citep{touvron2021training,hu2022lora,Kirillov_2023_ICCV,dosovitskiy2020image}. For example, ViT adapter under the plain ViT models has been successfully applied in multiple computer vision tasks, \eg,
% image classification~\citep{jie2023fact}, 
object detection (ODet)~\citep{li2022exploring}, instance segmentation (ISeg)~\citep{liu2024revisiting}, and semantic segmentation (SSeg)~\citep{xie2021segformer}.

Despite the significant progress made by existing ViT adapters~\citep{chen2022vision,liu2024revisiting,jie2023fact,marouf2024mini}, their inference speed is still somewhat unfavorable, which limits their implementations on edge computing devices~\citep{dong2024packqvit} and applications for real-time recognition scenarios~\citep{marouf2024mini}. 
Recently, it has been revealed that in addition to computation and parameter complexity (\eg, FLOPs and \#Params.), inefficient memory access operations~\citep{marouf2024mini,liu2023efficientvit}, such as standard normalization and frequent reshaping operations, play a critical role in hindering the ViT's inference speed~\citep{he2023simplifying,pan2022fast,fournier2023practical,shi2023evit}. 
In other words, the inefficient planning of memory access may cause delays and prevent models from fully utilizing the computing power of GPUs or CPUs, resulting in a significant negative impact on the speed of ViT models~\citep{liu2023efficientvit,dao2022flashattention,venkat2019swirl,gu2021towards,ivanov2021data}. However, these inefficient memory access operations are usually overlooked factors during the network design process for dense prediction tasks.
In this work, we explore how to solve this problem and accelerate the inference speed of ViT adapters in downstream models. \emph{Our solution is to decrease memory time consumption by reducing layer normalization and frequent reshaping operations.}

% -----------------------------------
% \input{figures/figure0}
% -----------------------------------
\begin{wrapfigure}{r}{0.53\textwidth}
\centering
\vspace{-4mm}
\includegraphics[width=.53\textwidth]{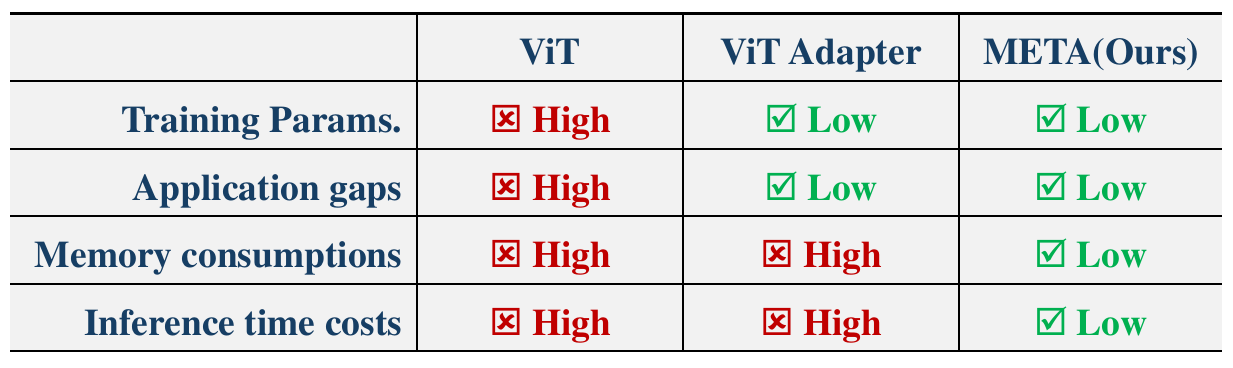}
\vspace{-8mm}
\caption{\footnotesize Qualitative performance comparisons of different models with respect to training parameters, application gaps, memory access costs, and inference time costs.}
\vspace{-4mm}
\label{fig0}
\end{wrapfigure}
% -----------------------------------
We propose a simple and fast memory efficient transformer adapter (META) that can improve the model's memory efficiency and decrease memory time consumption by reducing the inefficient memory access operations. 
As illustrated in Figure~\ref{fig0}, our META demonstrates advantages over existing ViT models and ViT adapters in terms of memory consumptions and inference time costs.
The main contribution of this work is the proposal of a \emph{{memory-efficient adapter block}} that shares {normalization operations} between the self-attention and feed-forward network layers, which exist in a parallel manner (\emph{Ref.} Sec.~\ref{sec:3:2}). Within this block, the cross-shaped self-attention is employed to reduce the reliance for frequent {reshaping operations}. 
Consequently, the proposed two reverse designs are capable of significantly reducing the memory consumption of the ViT adapter, resulting in an improved inference speed.
Moreover, to enrich local inductive biases for dense predictions, a \emph{{lightweight convolutional branch}} is introduced into the \emph{{memory-efficient adapter block}}. 
In the process of interacting with the ViT backbone, a \emph{{cascaded mechanism}} is further proposed to compute different head features, which can enhance the diversity of the obtained feature representations. We conduct extensive experiments for ODet, ISeg and SSeg on two challenging datasets, namely MS-COCO~\citep{lin2014microsoft} and ADE20K~\citep{zhou2017scene}, for evaluating our META. 
The obtained results demonstrate that META substantially enhances prediction quality, attains a new state-of-the-art accuracy level, reduces the number of parameters and memory consumption requirements, and achieves faster inference speeds. Theoretically, we also prove that META can exhibit superior generalization capability and stronger adaptability compared to existing ViT adapter methods\footnote{Due to page limitations, the theoretical analysis will been provided in the supplementary materials.}.
% ---------------------------------------------------
% -------------------------------
\section{Related Work}
% -------------------------------
\myparagraph{Vision Transformer (ViT).}
% -------------------------------
% ViT models have become one of the defacto network architectures for state-of-the-art computer vision tasks~\citep{han2022survey,zhang2022graph,liu2021swin,khan2022transformers,zhang2023cae,Kirillov_2023_ICCV}.
Current ViT models for state-of-the-art computer vision tasks can be roughly divided into two camps: \emph{plain ViT models}, which are designed to learn general vision features (\eg, ViT~\citep{dosovitskiy2020image}, DeiT~\citep{touvron2021training} and TDE Transformer~\citep{touvron2021training}), and \emph{hierarchical ViT models}, which are designed to learn vision-specific features (\eg, Swin Transformer~\citep{liu2021swin} and PVT~\citep{wang2021pyramid}). Both camps of ViT models have their pros and cons. 
For example, the \emph{hierarchical ViT models} can learn powerful vision-specific feature representations, which makes them usually perform better than the \emph{plain ViT models} on accuracy~\citep{han2022survey,khan2022transformers,Kirillov_2023_ICCV}. However, their unvarnished disadvantage is the lack of multi-modal pre-training information for dense predictions~\citep{chen2022vision}. On the other hand, although the \emph{plain ViT models} have weak prior assumptions, which results in lower accuracy compared to the \emph{hierarchical ViT models}, their multi-modal pre-training information has great potential to provide semantic-rich feature representations for downstream models~\citep{guo2022attention,zhang2023cae,he2022masked,zhang2024boundary,jia2022visual}. In this paper, we focus on the \emph{plain ViT models} for dense predictions. Our contribution is the proposal of a straightforward and memory efficient ViT adapter, starting from the reduction of inefficient memory access operations.

\myparagraph{ViT for Dense Predictions.}
% -------------------------------
Dense predictions, including ODet, ISeg, and SSeg, aim at predicting a semantic mask where each pixel/object in the given image is assigned a class label~\citep{ranftl2021vision,zhang2023cae,vandenhende2021multi}. For ISeg, the obtained mask also distinguishes different objects of the same class. 
% Dense predictions have been extensively studied in the past and have achieved remarkable results~\citep{chen2021pre,liu2023efficientvit,lee2022mpvit}. 
With the development of ViT technologies, dense prediction models based on ViT have become one of the default choices for state-of-the-art methods~\citep{liu2021swin,wang2021pyramid,dosovitskiy2020image}. These methods can be mainly divided into two categories: those based on Transformers and those that incorporate a combination of CNN layers. The former, \eg, SegFormer~\citep{strudel2021segmenter}, HRViT~\citep{gu2022multi}, Segvit~\citep{zhang2022segvit}, Swin UNet~\citep{hatamizadeh2021swin}, and ISTR~\citep{hu2021istr}, extract features from a given image and connect with task-specific head networks to achieve specific recognition purposes. The latter compensates for the lack of local inductive bias in the pure ViT architecture by introducing local convolutional layers, and representative methods include ConFormer~\citep{peng2021conformer}, CVT~\citep{wu2021cvt}, NextViT~\citep{li2022next}, and CAE-GReaT~\citep{zhang2023cae}. While the former can capture long-range dependencies and achieve favorable recognition capability, they require sufficient training samples to optimize the model, and further improvement is needed on memory access and inference speed. To this end, we propose a solution to the problem of limited inference speed caused by frequent memory consumption using a pre-trained ViT model.%  Our main objective is to achieve memory efficiency by reducing inefficient memory access operations.

% -------------------------------
\myparagraph{ViT Adapters.}
% -------------------------------
Adapter methods are originally derived from the NLP tasks, such as PALs~\citep{stickland2019bert} and NLP adapter~\citep{houlsby2019parameter}. These methods add a small number of trainable modules to a pre-trained network, so that downstream fine-tuning models can quickly adapt to specific datasets and tasks via a parameter-efficient transfer learning manner~\citep{hu2022lora,houlsby2019parameter,chen2022vision,shen2023survey}. In the computer vision domain, ViT adapters have also adopted the similar paradigm~\citep{jie2023fact,chen2022vision,shao2023deepfake,luo2023forgery}. 
% For example, in the training process, the ViT adapter branch first encodes the input image into a set of semantic features through a trainable spatial prior module, and then extracts the general features of the ViT backbone network through the spatial feature extractor, and injects the learned vision-specific features into the backbone network via the spatial feature injector~\citep{chen2022vision}. 
ViT adapter method is used to fine-tune large-scale plain ViT models with a small number of trainable parameters and achieve promising accuracy~\citep{jie2023fact,li2022exploring,chen2022vision}. 
% Currently, this technology has been widely used in many fundamental computer vision tasks, such as image classification~\citep{jie2023fact}, object detection~\citep{li2022exploring}, forgery detection~\citep{luo2023forgery}, and dense predictions~\citep{chen2022vision}. 
Inspired by~\citep{he2023simplifying,jie2022convolutional,chen2022vision,liu2023efficientvit,mercea2024time}, we propose an efficient solution to enhance the memory efficiency and minimize memory time consumption of ViT adapters. We also provide theoretical analysis to demonstrate the superior generalization and adaptability of our method.

% ---------------------------------------------------
% -------------------------------
\section{Memory Efficient Transformer Adapter (META)}

\subsection{Overall Architecture and Configurations}
% -------------------------------
As illustrated in Figure~\ref{fig1}, the network takes an arbitrary RGB image $\textbf{I} \in \mathbb{R}^{H \times W \times 3}$ as input and predicts a semantic mask $\textbf{O} \in \mathbb{R}^{H \times W \times C}$ as output. $H$ and $W$ denotes the image height and width, respectively, and $C$ denotes the class size of the used dataset.
Following~\citep{jie2023fact,li2022exploring,chen2022vision}, the whole network mainly consists of two parts: the upper part is a pre-trained plain ViT model (\eg, ViT~\citep{li2021benchmarking}), which consists of $4$ blocks. The $i$-th ($i = 1, 2, 3, 4$) block contains a set of semantic features $\textbf{F}^i_{vit}$ with the spatial size of $1/4$, $1/8$, $1/16$, and $1/32$ of $\textbf{I}$, respectively. 
The lower part is a trainable ViT adapter, which includes a spatial prior module as in~\citep{chen2022vision} for $\textbf{I}$'s encoding, where the encoded features within each block denote $\textbf{F}^i_{sp}$, as well as a set of cascaded memory efficient adapter (MEA) injectors (\emph{Ref.}~Sec.~\ref{sec:3:3}) and cascaded MEA extractors (\emph{Ref.}~Sec.~\ref{sec:3:4}) that act on each ViT backbone block. 
%
% As shown in Figure~\ref{fig1} (a), our main contribution of this work is the proposal of the MEA block, which serves as the foundation for the injector and the extractor. 
%
The ``cascaded'' refers to a computational scheme that is incorporated into the different heads of the proposed injectors/extractors, which can enhance the diversity of the obtained feature representations.
%
% The role of the injector and the extractor is to facilitate feature interactions between the ViT adapter and the pre-trained ViT backbone by operating on each block. 
%
Concretely, for the $i$-th trainable MEA injector, it injects the updated task-specific features back into the $i$-th ViT block and into the $i$-th extractor (as shown in Figure~\ref{fig1} (b)). For the $i$-th trainable MEA extractor, it extracts the $i$-th general features for the $(i+1)$-th injector (as shown in Figure~\ref{fig1} (c)). This process is repeated until the final ViT block is reached. 
%
% In training, only the adapter's parameters need to be updated, while the ViT's parameters remain fixed. 
%
Due to the incorporation of strategies aimed at reducing inefficient memory access operations in our method, which exhibits enhanced memory efficiency and reduced memory time consumption compared to existing ViT adapters, consequently leading to accelerated inference speed.
We will provide a detailed explanation of the basic component and mechanism of the MEA block in the following Sec.~\ref{sec:3:2} and how it can be deployed in the MEA injectors and MEA extractors to form a cascaded scheme.
% -------------------------------
\begin{figure*}[t]
\small
\centering
\includegraphics[width=1\textwidth]{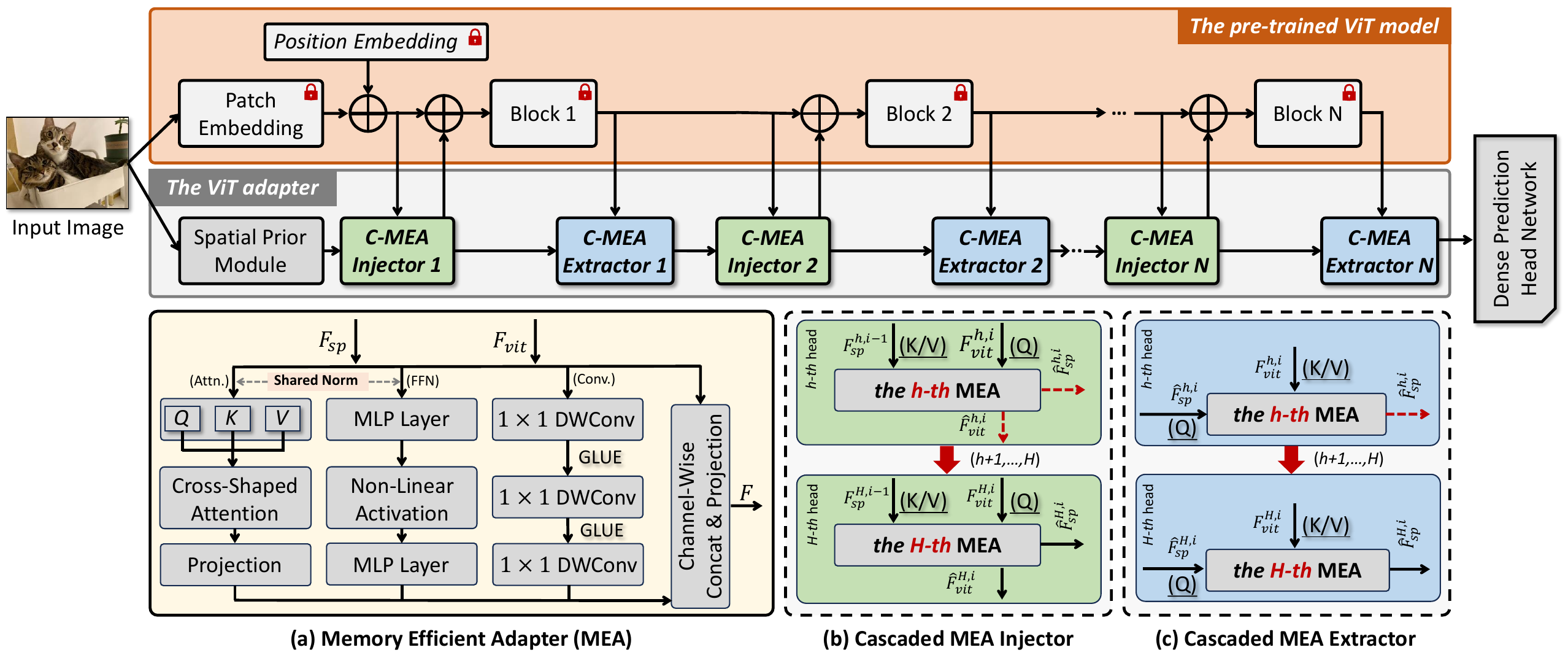}
\vspace{-8mm}
\caption{Overall architecture of META. Our primary contribution is the proposal of a MEA block in (a), which serves as the fundamental component for the injector in (b) and extractor in (c).}
\vspace{-3mm}
\label{fig1}
\end{figure*}
% -------------------------------

% -------------------------------
\subsection{MEA Block}
\label{sec:3:2}
% -------------------------------
The MEA block aims to interact the features extracted from the ViT backbone and the spatial prior module. Our block consists of the attention (\ie, Attn) branch, the feed-forward network (\ie, FFN) branch, and the lightweight convolutional (\ie, Conv) branch. The three branches exist in a parallel form, which is beneficial for computing in GPUs.
As shown in Figure~\ref{fig1} (a), consider an arbitrary block, the input is $\textbf{F}_{sp}$ and $\textbf{F}_{vit}$, and the output $\textbf{F}$ is obtained by concatenating the features from the Attn, FFN, and Conv branches along the channel dimension, and then passing them through a feature projection layer. Therefore, this process can be formulated as:
% -------------------------------
\begin{equation}
\textbf{F} = \textrm{Conv}_{3 \times 3}(\textrm{Concat}(
\underbrace{(\textrm{A}(\textbf{F}_{sp},\textbf{F}_{vit})}_{\textrm{\textbf{\textcolor{red}{Attn Branch}}}};
\underbrace{\textrm{F}(\textbf{F}_{sp},\textbf{F}_{vit})}_{\textrm{\textbf{\textcolor{blue}{FFN Branch}}}};
\underbrace{\textrm{C}(\textbf{F}_{sp},\textbf{F}_{vit})}_{\textrm{\textbf{\textcolor{orange}{Conv Branch}}}};\textbf{F}_{sp};\textbf{F}_{vit})),
\label{eq:1}
\end{equation}
% -------------------------------
where $\textrm{A} (\cdot)$, $\textrm{F} (\cdot)$, and $\textrm{C} (\cdot)$ denote the operation of the Atte, FFN, and Conv branch, respectively. In MEA, these three branches exist in a parallel manner. $\textrm{Concat} (\cdot)$ denotes the feature concatenation operation along the channel dimension, and $\textrm{Conv}_{3 \times 3}(\cdot)$ denotes a $3 \times 3$ convolution with the output channel size of $256$ for feature projection. 
%In our implementation, we also concatenate with $\textbf{F}_{sp}$ and $\textbf{F}_{vit}$ in $\textrm{Concat} (\cdot)$. For simplicity, this part is omitted in~\eqref{eq:1}. 
In particular, to construct an memory efficient block, following~\citep{he2023simplifying,wang2021gpt}, \emph{the Attn and FFN branches are subjected to a shared layer normalization operation, which leads to a decrease in memory time consumption associated with the normalization operations}.

% -------------------------------
\myparagraph{\textcolor{red}{Attn Branch.}} In our work, we adopt the cross-shaped self-attention (CSA) to reduce the model's access to the memory unit, which allows for \emph{the computation of the attention matrix across different spatial dimensions without requiring the input tensor to be frequently reshaped}~\citep{dong2022cswin,tu2022maxvit}. 
Specifically, CSA first performs self-attention separately along the horizontal and vertical dimensions of the given features. The outputs $\textrm{A}_H(\textbf{F}_{sp},\textbf{F}_{vit})$ and $\textrm{A}_V(\textbf{F}_{sp},\textbf{F}_{vit})$ of these two parallel groups are then concatenated along the channel dimension, followed by a feature projection operation via a $3 \times 3$ convolution with the output channel size of $256$, to form the output $\textrm{A}(\textbf{F}_{sp},\textbf{F}_{vit})$. 
For the horizontal self-attention, the input features $\textbf{F}_{sp}$ and $\textbf{F}_{vit}$ are first subjected to the shared layer normalization, and then divided into $M$ non-overlapping horizontal stripes, where each stripe has the spatial size of $s \times W$ (\ie, $s = H/M$). Then, the self-attention~\citep{vaswani2017attention} is performed on each stripe, which can be formulated as:
% -------------------------------
\begin{equation}
\textrm{A}^{m}_H(\textbf{F}_{sp},\textbf{F}_{vit}) = \textrm{SA}(\textrm{LN}(\textbf{F}^{m}_{sp}W^Q),\textrm{LN}(\textbf{F}^{m}_{vit}W^K),\textrm{LN}(\textbf{F}^{m}_{vit}W^V)),
\label{eq:2}
\end{equation}
% -------------------------------
where $\textrm{LN}(\cdot)$ denotes the shared layer normalization operation. $\textrm{SA}(\cdot)$ denotes the classical single head self-attention~\citep{vaswani2017attention}. $m$ denotes the $m$-th stripe and $m = 1,2,...,M$. $W^Q$, $W^K$, and $W^V$ are used to project the input features into the queries, keys, and values spaces, respectively. The sub-attention results of different stripes are merged together to form the output $\textrm{A}_H(\textbf{F}_{sp},\textbf{F}_{vit})$ of the horizontal group. In~\eqref{eq:2}, we choose $\textbf{F}_{sp}$ as queries and $\textbf{F}_{vit}$ as keys and values for the sake of example. In our injector and extractor in Sec~\ref{sec:3:3} and Sec~\ref{sec:3:4}, the roles of $\textbf{F}_{sp}$ and $\textbf{F}_{vit}$ can be swapped depending on the specific requirements. The attention computation along the vertical dimension $\textrm{A}_V(\textbf{F}_{sp},\textbf{F}_{vit})$ is also performed in a similar fashion.

% -------------------------------
\myparagraph{\textcolor{blue}{FFN Branch.}} 
% -------------------------------
% FFN branch is to enable interaction between the obtained features in the channel dimension, thereby enhancing the representation capacity of the features
Following the common setting~\citep{vaswani2017attention,tolstikhin2021mlp}, to enable interaction between the obtained features in the channel dimension, our FFN branch consists of two MLP layers and a non-linear activation layer~\citep{saxe2013exact,he2023simplifying}, where the MLP layer consists of a sequential arrangement of two $3 \times 3$ convolutional layers.

The input of the FFN branch is $\textbf{F}_{sp}$ and $\textbf{F}_{vit}$, which are sequentially processed by feature concatenation along the channel dimension, a $3 \times 3$ convolution with the output channel size of $256$, and the shared layer normalization. This process can be expressed as:
% -------------------------------
\begin{equation}
\textrm{F}(\textbf{F}_{sp},\textbf{F}_{vit})_{\textrm{Tem}} = \textrm{LN}(\textrm{Conv}_{3 \times 3}(\textrm{Concat}(\textbf{F}_{sp};\textbf{F}_{vit}))).
\label{eq:3}
\end{equation}
% -------------------------------
The temporary $\textrm{F}(\textbf{F}_{sp},\textbf{F}_{vit})_{\textrm{Tem}}$ are then passed through a MLP layer, a non-linear activation layer~\citep{saxe2013exact}, and another MLP layer to obtain the output $\textrm{F}(\textbf{F}_{sp},\textbf{F}_{vit})$. Hence, this can be expressed as:
% -------------------------------
\begin{equation}
\textrm{F}(\textbf{F}_{sp},\textbf{F}_{vit}) = \textrm{MLP}(\textrm{NLA}(\textrm{MLP}(\textrm{F}(\textbf{F}_{sp},\textbf{F}_{vit})_{\textrm{Tem}}))),
\label{eq:4}
\end{equation}
% -------------------------------
where $\textrm{MLP}(\cdot)$ denotes the MLP layer operation, and $\textrm{NLA}(\cdot)$ denotes the non-linear activation layer~\citep{saxe2013exact,he2023simplifying}.

% -------------------------------
\myparagraph{\textcolor{orange}{Conv Branch.}} 
% -------------------------------
Empirical evidence has demonstrated that incorporating local inductive biases into ViT models can provide notable benefits for visual tasks involving dense prediction~\citep{peng2021conformer,wu2021cvt,li2022next,zhang2023cae}. To achieve this goal, we introduce a lightweight convolutional branch into the ViT adapter. Specifically, the Conv branch also takes into $\textbf{F}_{sp}$ and $\textbf{F}_{vit}$ as the input, which consists of three $1 \times 1$ depth-wise convolutions concatenated in series, with a GLUE layer used to activate the features between every two convolutions~\citep{saxe2013exact,he2023simplifying}. The computation process of the Conv branch is formulated as:
% -------------------------------
\begin{equation}
\textrm{C}(\textbf{F}_{sp},\textbf{F}_{vit}) = \textrm{DC}(\textrm{GLU}(\textrm{DC}(\textrm{GLU}(\textrm{DC}(\textrm{Concat}(\textbf{F}_{sp};\textbf{F}_{vit}))))),
\label{eq:5}
\end{equation}
% -------------------------------
where $\textrm{DC}(\cdot)$ denotes the $1 \times 1$ depth-wise convolution with the output channel size of $256$, and $\textrm{GLU}(\cdot)$ denotes the GLUE layer.
% -------------------------------
With the help of shared layer normalization and CSA, MEA becomes a memory-efficient module with rich local inductive biases. In the following sections, we will provide detailed instructions on how to deploy our MEA block on injectors and extractors.
% -------------------------------
\subsection{Cascaded MEA Injector}
\label{sec:3:3}
% -------------------------------
% The injector aims at transmitting the updated task-specific features into the ViT backbone and into the extractor after feature interaction~\citep{marouf2024mini,chen2022adaptformer,dong2024efficient}. 
As illustrated in Figure~\ref{fig1} (b), the cascaded MEA injector adopts the paradigm of self-attention in its computational process, where $\textbf{F}^i_{vit}$ is used as the query, and $\textbf{F}^{i-1}_{sp}$ generated by the last extractor is used as the key and value. 
Particularly, to enhance the diversity of the obtained attention maps, a cascaded mechanism is proposed to compute different head features in the cascaded MEA injector. Specifically, we first divide the given features into $H$ parts along the channel dimension in a multi-head manner, following the classic self-attention~\citep{vaswani2017attention}, where $H$ is set to 16 in our work.
In the computation process of each head, the output of the $h$-th head $\hat{\textbf{F}}^{h,i}_{sp}$ and $\hat{\textbf{F}}^{h,i}_{vit}$ is added into the input features of the next ($h+1$)-th head $\textbf{F}^{h+1,i-1}_{sp}$ and $\textbf{F}^{h+1,i}_{vit}$ to be used in the calculation of subsequent self-attention features, where $h =1, 2, ..., H$. The cascaded process continues until the feature from the last head is included in the computation. Finally, the features obtained from these heads are concatenated along the channel dimension and projected through a $3 \times 3$ convolution layer before being outputted as the output of the $i$-th cascaded MEA injector. Besides, since the MEA block has only one output, $\hat{\textbf{F}}^{h,i}_{sp} = \hat{\textbf{F}}^{h,i}_{vit}$ in our cascaded MEA injector.

% -------------------------------
\subsection{Cascaded MEA Extractor}
\label{sec:3:4}
% -------------------------------
As illustrated in Figure~\ref{fig1} (c), following~\citep{marouf2024mini,chen2022adaptformer,dong2024efficient},
% the extractor is to extract general features for the ViT block and to interact with the features output by the injector before inputting them into the injector of the next block~\citep{marouf2024mini,chen2022adaptformer,dong2024efficient}. Therefore, the 
our $i$-th cascaded MEA extractor of the $h$-th head takes the $\hat{\textbf{F}}^{h,i}_{sp}$ generated by the injector and $\textbf{F}^{h,i}_{vit}$ as input, where $\hat{\textbf{F}}^{h,i}_{sp}$ is used as the query, $\textbf{F}^{h,i}_{vit}$ is used as the key and value. The output of the $i$-th cascaded MEA extractor is $\textbf{F}^{h,i}_{sp}$. Similarly, the cascaded mechanism is applied following the cascaded MEA injector until the feature from the last head is included in the computation.
\section{Experiments}
% -------------------------------
\subsection{Datasets and Evaluation Metrics}
\label{sec4:1}
% -------------------------------
\myparagraph{Datasets.} To facilitate a fair result comparison with existing methods, we conduct experiments, including the ablation analysis, on two commonly used datasets: MS-COCO~\citep{caesar2018coco} for ODet and ISeg, and ADE20K~\citep{zhou2017scene} for SSeg. Due to page limitations, the implementation details of these datasets will be given in the supplementary materials. 
% Below are the details of the used datasets:

% -------------------------------
% \begin{itemize}
% -------------------------------
% \item MS-COCO~\citep{caesar2018coco} is a representative yet challenging dataset for common scene IS and object detection, which consists of $118$k, $5$k and $20$k images for the \emph{training} set, the \emph{val} set and the \emph{test} set, respectively. In our experiments, the model is trained on the \emph{training} set and evaluated on the \emph{val} set.
% -------------------------------
% \item ADE20K~\citep{zhou2017scene} is a scene parsing dataset with $20$k images and $150$ object categories. Each image has pixel-level annotations for SS of objects and regions within the scene. The dataset is divided into $20$k, $2$k, and $3$k images for \emph{training}, \emph{val} and \emph{test}, respectively. Our model is trained on the \emph{training} set and evaluated on the \emph{val} set.
% -------------------------------
% \end{itemize}
% -------------------------------
% For data augmentation, random horizontal flip, brightness jittering and random scaling within the range of $[0.5, 2]$ are used in training as in~\citep{chen2022vision,luo2023forgery,zhang2023cae}. By default, the inference results are obtained at a single scale, unless explicitly specified otherwise.    

\myparagraph{Evaluation metrics.} The commonly adopted average precision (AP) and mean intersection-over-union (mIoU) are used to assess the model accuracy for ODet (AP$^\textrm{b}$)/ISeg (AP$^\textrm{m}$) and SSeg, respectively. Besides, to evaluate the efficiency, model Parameters (\#P), floating point operations (FLOPs), memory consumption (MC) of the adapter, and frames per second (FPS) are also adopted. The reported inference results are measured by A100 GPUs with per-GPU batch size 2.
%
% \emph{We acknowledge that using memory access cost can provide an intuitive reflection of a model's memory efficiency. However, since this metric is seldom reported in publicly available papers, we refrain from comparing this metric to ensure a fair comparison. Instead, we use FPS as the evaluation metric (Ref Sec.~\ref{sec4:4}), which has a strong correlation with the model's memory access cost}~\citep{liu2023efficientvit,dao2022flashattention,gu2021towards}\footnote{Due to page limitations, we only provide the key quantitative results. Qualitative results and more ablation studies will be presented in the supplementary material.}.  
% -------------------------------
% -------------------------------
\subsection{Experiments on Object Detection (ODet) and Instance Segmentation (ISeg)}
\label{sec4:2}
% -------------------------------
\myparagraph{Baselines and settings.} As in~\citep{chen2022vision,xiong2024efficient,jie2022convolutional,marouf2024mini}, Mask R-CNN~\citep{he2017mask}, Cascade Mask R-CNN~\citep{cai2019cascade}, ATSS~\citep{zhang2020bridging}, and GFL~\citep{li2020generalized} are employed as the baseline models, where the pre-trained ViT~\citep{li2022exploring} is used as the backbone. All the baseline models are pre-trained on ImageNet-1k by default~\citep{deng2009imagenet}. Unless otherwise specified, these baselines are set up to be consistent with their papers and the settings of the ViT-Adapter~\citep{chen2022vision} method.
% -------------------------------
\begin{table*}[t]
\centering
\small
\renewcommand\arraystretch{1.1}
\setlength{\tabcolsep}{1.2pt}{
\begin{tabular}{l|ccc|cccc|cccc}
\hline \hline 
\multirow{2}{*}{Methods} & \#P & FLOPs & MC & \multicolumn{4}{c|}{ $1 \times$ schedule} & \multicolumn{4}{c}{$3 \times$ +MS schedule} \\
~ & \begin{small}(\textbf{M})\end{small} & \begin{small}(\textbf{G})\end{small} & \begin{small}(\textbf{GB})\end{small} & AP$^\textrm{b}$  & AP$^\textrm{m}$  & AP$^\textrm{m}$$_{50}$  & AP$^\textrm{m}$$_{75}$  & AP$^\textrm{b}$ & AP$^\textrm{m}$  & AP$^\textrm{m}$$_{50}$ & AP$^\textrm{m}$$_{75}$  \\
\hline 
% PVT-Tiny~\citep{wang2021pyramid} & 32.9 &  &   &  &  35.1 & 56.7 & 37.3 &  &  37.4 & 59.3 & 39.9 \\
% PVTv2-B1~\citep{wang2022pvt} & 33.7 & 211 & NA & 41.8 & 38.8 & 61.2 & 41.6 & 44.9 & 40.8 & 64.0 & 43.8 \\
ViT-T~\citep{li2021benchmarking} & 26.1 & 205 & NA & 35.5 & 33.5 & 54.9 & 35.1 & 44.9 & 37.0 & 59.6 & 39.0 \\ 
ViTDet-T~\citep{li2022exploring} & 26.6 & 209 & NA & 35.7 & 33.5 & 54.7 & 35.2 & 40.4 & 37.1 & 60.1 & 39.3 \\
\cdashline{1-12}[0.8pt/2pt]
ViT-Adapter-T~\citep{chen2022vision} & 28.1 & 208 & 15.2 & 41.1 & 37.5 & 59.7 & 39.9 & 46.0 & 41.0 & 64.4 & 44.1 \\
AdaptFormer-T~\citep{chen2022adaptformer} & 27.9 & 210 & 17.7 & 36.4 & 34.2 & 55.6 & 35.9 & 45.2 & 37.5 & 60.5 & 40.8\\
FacT-TK-T~\citep{jie2023fact} & 27.7 & 213 & 16.1 & 36.6 & 34.5 & 55.5 & 36.7 & 46.3 & 41.2 & 64.7 & 44.0\\
LoSA-T~\citep{mercea2024time} & 27.6 & 208 & 13.0 & 41.4 & 38.0 & 60.1 & 40.5 & 46.1 & 39.5 & 64.2 & 44.3\\
\cellcolor[gray]{.95}\textbf{META-T$_{{\textrm{(Ours)}}}$} & \cellcolor[gray]{.95}27.4 & \cellcolor[gray]{.95}206 & \cellcolor[gray]{.95}8.1 & \cellcolor[gray]{.95}43.1& \cellcolor[gray]{.95}38.6  & \cellcolor[gray]{.95}60.8 & \cellcolor[gray]{.95}41.1  & \cellcolor[gray]{.95}47.5 & \cellcolor[gray]{.95}41.6  & \cellcolor[gray]{.95}65.0 & \cellcolor[gray]{.95}44.8 \\
\hline 
% PVT-Small~\citep{wang2021pyramid} \\ 
% PVTv2-B2~\citep{wang2022pvt} \\
% Swin-T~\citep{liu2021swin} & 47.8 & 267 &   &  &  39.3 & 62.2 & 42.2 &  & 41.6 & 65.1 & 44.9 \\
% ConvNeXt-T~\citep{liu2022convnet} & 48.1 & 262 & NA & 44.2 & 40.1 & 63.3  &  42.8 & 46.2 &  41.7 & 65.0 & 44.9 \\
% Focal-T~\citep{yang2021focal} & 48.8 & 41.0 & 64.7 & 44.2 & 42.7 & 66.5 & 45.9 \\
ViT-S~\citep{li2021benchmarking} & 43.8 & 245 & NA & 40.2 & 37.1 & 59.9 & 39.3 & 44.0 &  39.9 & 63.4 & 42.2 \\
ViTDet-S~\citep{li2022exploring} & 43.8 & 255  & NA & 40.6 & 37.1 & 60.0 & 38.8 & 44.5  &  40.1 & 63.6 & 42.5 \\
\cdashline{1-12}[0.8pt/2pt]
ViT-Adapter-S~\citep{chen2022vision} & 47.8 & 251 & 15.2 & 44.7 & 39.9 & 62.5 & 42.8 & 48.2 & 42.8 & 66.4 & 45.9 \\
AdaptFormer-S~\citep{chen2022adaptformer} & 47.5 & 255 & 17.7 & 41.2 & 37.8 & 60.3 & 41.0 & 45.0 & 41.2 & 64.2 & 43.4 \\
FacT-TK-S~\citep{jie2023fact} & 46.8 & 259 & 16.1 & 41.5 & 37.8 & 60.7 & 41.2 & 45.2 & 41.8 & 65.0 & 42.9 \\
LoSA-S~\citep{mercea2024time} & 46.1 & 249 & 13.0 & 44.2 & 39.8 & 62.6 & 42.5 & 48.3 & 42.5 & 66.5 & 46.0 \\
\cellcolor[gray]{.95}\textbf{META-S$_{{\textrm{(Ours)}}}$} & \cellcolor[gray]{.95}45.2 & \cellcolor[gray]{.95}247 & \cellcolor[gray]{.95}8.1  & \cellcolor[gray]{.95}45.5 & \cellcolor[gray]{.95}40.5 & \cellcolor[gray]{.95}63.4 & \cellcolor[gray]{.95}43.5 & \cellcolor[gray]{.95}49.5 & \cellcolor[gray]{.95}43.3 &  \cellcolor[gray]{.95}66.9 & \cellcolor[gray]{.95}46.6 \\
\hline 
% PVTv2-B5~\citep{wang2022pvt} & 101.6& &   &  &  42.5 & 65.7 & 46.0 &  & 42.9 & 66.6 & 46.2 \\
% Swin-B~\citep{liu2021swin} & 107.1 & 982 & NA & 46.9 & 42.3 & -- & -- & 48.6 & 43.3 & 67.1 & 46.7 \\
ViT-B~\citep{li2021benchmarking} & 113.6 & 719 & NA & 42.9 &  39.4 & 62.6 & 42.0 & 45.8 & 41.3 & 65.1 & 44.4 \\
ViTDet-B~\citep{li2022exploring} & 121.3 & 800 & NA & 43.2 &  39.2 & 62.7 & 41.4 & 46.3 & 41.6 & 65.3 & 44.5 \\
\cdashline{1-12}[0.8pt/2pt]
ViT-Adapter-B~\citep{chen2022vision} & 120.2 & 730 & 15.2 & 47.0 &  41.8 & 65.1 & 44.9 & 49.6 & 43.6 & 67.7 & 46.9 \\
AdaptFormer-B~\citep{chen2022adaptformer} & 118.5 & 733 & 17.7 &  44.5 & 40.3 & 63.1 & 43.0 & 46.2 & 42.4 & 66.2 & 45.7\\
FacT-TK-B~\citep{jie2023fact} & 118.0 & 735 & 16.1 & 44.9 & 41.2 & 64.1 & 43.2 & 47.6 & 43.2 & 67.5 & 46.3  \\
LoSA-B~\citep{mercea2024time} & 117.2 & 722 & 13.0 & 45.1 & 41.8 & 64.6 & 44.0 & 48.6 & 43.8 & 67.9 & 47.0 \\
\cellcolor[gray]{.95}\textbf{META-B$_{{\textrm{(Ours)}}}$} & \cellcolor[gray]{.95}115.3 & \cellcolor[gray]{.95}720  & \cellcolor[gray]{.95}8.1  & \cellcolor[gray]{.95}45.4 & \cellcolor[gray]{.95}42.3 & \cellcolor[gray]{.95}66.0 & \cellcolor[gray]{.95}45.7 & \cellcolor[gray]{.95}51.2 & \cellcolor[gray]{.95}44.3 & \cellcolor[gray]{.95}68.2 & \cellcolor[gray]{.95}47.5 \\
\hline 
ViT-L$^\ddag$~\citep{li2021benchmarking} & 337.3 & 1,907 & NA & 45.7 & 41.5 & 65.6 & 44.6 & 48.3 & 43.4 & 67.9 & 46.6 \\
ViTDet-L$^\ddag$~\citep{li2022exploring} & 350.9 & 1,900 & NA & 46.2 & 41.4 & 65.8 & 44.1 & 49.1 & 44.0 &  68.5 & 47.6\\
\cdashline{1-12}[0.8pt/2pt]
ViT-Adapter-L$^\ddag$~\citep{chen2022vision} & 347.9 & 1,927 & 15.2 & 48.7 & 43.3 & 67.0 & 46.9 & 52.1 & 46.0 & 70.5 & 49.7 \\
AdaptFormer-L$^\ddag$~\citep{chen2022adaptformer} & 345.7 & 1,924 & 17.7 & 46.3 & 42.6 & 66.1 & 45.8 & 49.0 & 44.8 & 68.1 & 47.1\\
FacT-TK-L$^\ddag$~\citep{jie2023fact} & 342.1 & 1,925 & 16.1 & 46.5 & 42.9 & 66.5 & 46.0 & 48.8 & 44.5 & 68.2 & 47.0 \\
LoSA-L$^\ddag$~\citep{mercea2024time} & 341.3 & 1,922 & 13.0 & 50.2 & 43.5 & 66.8 & 47.1 & 53.2 & 46.2 & 71.0 & 48.9 \\
\cellcolor[gray]{.95}\textbf{META-L$^\ddag$$_{{\textrm{(Ours)}}}$} & \cellcolor[gray]{.95}339.7 & \cellcolor[gray]{.95}1,905 & \cellcolor[gray]{.95}8.1 & \cellcolor[gray]{.95}51.8 &  \cellcolor[gray]{.95}44.1 & \cellcolor[gray]{.95}67.6 & \cellcolor[gray]{.95}47.8 & \cellcolor[gray]{.95}55.3 & \cellcolor[gray]{.95}46.7 & \cellcolor[gray]{.95}71.3 & \cellcolor[gray]{.95}50.4 \\
\hline \hline 
\end{tabular}
\vspace{-3mm}
\caption{Result comparisons with SOTA methods under Mask R-CNN~\citep{he2017mask} on the \emph{val} set of MS-COCO. $\ddag$ denotes the model is pre-trained on ImageNet-22k as in~\citep{steiner2021train}. ``MS'' denotes the multi-scale training strategy. ``NA'' denotes not applicable.}
\label{tab1}}
\vspace{-4mm}
\end{table*}
% -------------------------------

% -------------------------------
\myparagraph{Comparisons with state-of-the-art (SOTA) methods.} Result comparisons with SOTA methods with Mask R-CNN~\citep{he2017mask} for ODet and ISeg are shown in Table~\ref{tab1}. From this table, we can obtain the following observations and conclusions: 
\emph{\textbf{\romannumeral1}})~Compared to the experimental results of the Mask R-CNN~\citep{he2017mask} model with ViT~\citep{li2021benchmarking}, our proposed META can consistently improve accuracy for ODet and ISeg across different model scales (\eg, ViT-T/S/B/L~\citep{li2021benchmarking}), while only adding a small number of training parameters. Even with different training schedules (\ie, 1$\times$, and 3$\times$ with MS), our method can also improve the model performance, demonstrating the plug-and-play advantage of META. 
\emph{\textbf{\romannumeral2}})~Compared to the SOTA ViT-Adapter~\citep{chen2022vision} and LoSA~\citep{mercea2024time}, META can achieve a new accuracy-efficiency trade-off. For example, in settings with the strong ViT-B~\citep{li2021benchmarking} as the backbone, our method achieves a performance gain of $0.5\%$AP$^\textrm{m}$ under 1$\times$ training schedule and $0.7\%$AP$^\textrm{m}$ under 3$\times$ training schedule with MS while reducing 4.9\textbf{M} model parameters, when compared to ViT-Adapter~\citep{chen2022vision}. 
\emph{\textbf{\romannumeral3}})~Even with stronger pre-trained models, META still improves performance on the baseline models and surpasses existing methods on accuracy, parameters and memory. For example, on the ImageNet-22k pre-trained weights from~\citep{steiner2021train}, our method achieves a performance gain of $0.8\%$/$0.7\%$AP$^\textrm{m}$ under 1$\times$/3$\times$ training schedule while reducing 8.2\textbf{M} parameters compared to ViT-Adapter~\citep{chen2022vision}, which validates its strong learning ability and flexible adaptability.
\emph{\textbf{\romannumeral4}})~Compared to SOTA ODet and ISeg methods, META also has very competitive performance. For example, compared to the strong ViTDet-B~\citep{li2022exploring} model, META-T has $2.2\%$AP$^\textrm{b}$ and $3.1\%$AP$^\textrm{m}$ gains under the 1$\times$ training schedule, and has META-T has $4.9\%$AP$^\textrm{b}$ and $2.7\%$AP$^\textrm{m}$ gains under the 3$\times$ training schedule.
\emph{\textbf{\romannumeral5}})~Compared to SOTA ViT adapter methods such as AdaptFormer~\citep{chen2022adaptformer}, FacT-TK~\citep{jie2023fact}, and LoSA~\citep{mercea2024time}, our method exhibits superior efficiency in terms of reduced parameter count, decreased FLOPs, and lower MC. Particularly, META only utilizes $62\%$ of the MC of LoSA while achieving superior prediction accuracy. Therefore, the experimental results and conclusions above demonstrate that our method achieves better accuracy and higher efficiency in dense prediction tasks.
% -------------------------------
\begin{table}[t]
\centering
\small
\renewcommand\arraystretch{1.1}
\setlength{\tabcolsep}{6.2pt}{
\begin{tabular}{l|ccc|c|llll}
\hline \hline 
{Methods} & \#P & FLOPs & MC & {FPS} & {AP$^\textrm{b}$} & {AP$^\textrm{m}$} & {AP$^\textrm{m}$$_{50}$} & {AP$^\textrm{m}$$_{75}$}\\
\hline 
\multicolumn{9}{c}{Cascade Mask R-CNN~\citep{cai2019cascade} $3 \times$ +MS schedule} \\
\hline 
Swin-T~\citep{liu2021swin} & 86 & 745 & NA & 15.3 & 50.5 & 43.7 & 66.6 & 47.1  \\
% Shuffle-T~\citep{huang2021shuffle} & 86& 746& 44.1 & 66.9 & 48.0 \\
% PVTv2-B2~\citep{wang2022pvt} & 83& 788& 44.2 & 65.5 & 48.5 \\
% Focal-T~\citep{yang2021focal} & 87& 770& 44.7 & 65.5  &48.9 \\
ViT-S~\citep{li2021benchmarking} & 82 & 795 & NA & 16.5 & 47.9 & 42.8 & 62.1 & 44.8 \\
\cdashline{1-9}[0.8pt/2pt]
ViT-Adapter-S~\citep{chen2022vision} & 86 & 801 & 15.2 & 13.1 & 51.5 & 43.5 & 63.0 & 48.2 \\
LoSA-S~\citep{mercea2024time} & 84 & 799 & 13.0 & 12.0 & 53.5 & 43.8 & 64.1 & 48.1 \\
\cellcolor[gray]{.95}\textbf{META-S$_{{\textrm{(Ours)}}}$} & \cellcolor[gray]{.95}83& \cellcolor[gray]{.95}797 & \cellcolor[gray]{.95}8.1 & \cellcolor[gray]{.95}14.7& \cellcolor[gray]{.95}54.3 & \cellcolor[gray]{.95}44.8 & \cellcolor[gray]{.95}66.8 & \cellcolor[gray]{.95}49.5 \\
\hline 
Swin-B~\citep{liu2021swin} & 145 & 982 & NA & 11.6 & 51.9 & 45.0 & 68.4 & 48.7 \\
% Shuffle-B~\citep{huang2021shuffle} & 145& 989& 45.3 & 68.5 & 48.9 \\
ViT-B~\citep{li2021benchmarking} & 151 & 1,100 & NA & 13.0 & 50.1 & 42.7 & 63.9 & 44.1 \\
\cdashline{1-9}[0.8pt/2pt]
ViT-Adapter-B~\citep{chen2022vision} & 158 & 1,106 & 15.2 & 8.6 & 52.1 & 43.6 & 64.3 & 45.2 \\
LoSA-B~\citep{mercea2024time} & 155 & 1,104 & 13.0 & 9.0 & 52.6 & 43.5 & 63.8 & 44.7  \\
\cellcolor[gray]{.95}\textbf{META-B$_{{\textrm{(Ours)}}}$} & \cellcolor[gray]{.95}153& \cellcolor[gray]{.95}1,101 & \cellcolor[gray]{.95}8.1 & \cellcolor[gray]{.95}11.9 & \cellcolor[gray]{.95}53.8 & \cellcolor[gray]{.95}44.1 & \cellcolor[gray]{.95}65.2 & \cellcolor[gray]{.95}45.6 \\
\hline 
\multicolumn{9}{c}{ATSS~\citep{zhang2020bridging} $3 \times$ +MS schedule} \\
\hline 
Swin-T~\citep{liu2021swin} & 36 & 215 & NA & 17.1 & 47.2 & 41.2 & 54.8 & 45.5 \\
% Focal-T~\citep{yang2021focal} & 37& 239& 42.6 & 55.0 & 46.1 \\
% PVTv2-B2~\citep{wang2022pvt} & 33& 258& &  & &  42.1 & 55.2 & 45.8 \\
ViT-S~\citep{li2021benchmarking} & 32 & 263 & NA & 18.0 & 45.2 & 40.5 & 52.0 & 41.8 \\
\cdashline{1-9}[0.8pt/2pt]
ViT-Adapter-S~\citep{chen2022vision} & 36& 272 & 15.2 & 14.2 & 49.6 & 42.5 & 55.1 & 46.5 \\
LoSA-S~\citep{mercea2024time} & 35 & 268 & 13.0 & 16.0 & 50.3 & 41.9 & 54.4 & 45.0  \\
\cellcolor[gray]{.95}\textbf{META-S$_{{\textrm{(Ours)}}}$} & \cellcolor[gray]{.95}33& \cellcolor[gray]{.95}265 & \cellcolor[gray]{.95}8.1 & \cellcolor[gray]{.95}16.8 & \cellcolor[gray]{.95}54.9 & \cellcolor[gray]{.95}43.2 & \cellcolor[gray]{.95}55.6 & \cellcolor[gray]{.95}47.4 \\
\hline  
\multicolumn{9}{c}{GFL~\citep{li2020generalized} $3 \times$ +MS schedule} \\
\hline 
Swin-T~\citep{liu2021swin} & 36 & 251 & NA & 17.5 & 47.6 & 40.8 & 54.2 & 46.0\\
% PVTv2-B2~\citep{wang2022pvt} & 33& 261 & & & &  42.1 & 56.0 & 47.2\\
ViT-S~\citep{li2021benchmarking} & 32 & 275 & NA & 18.1 & 46.0 & 40.2 & 52.9 & 45.1\\
\cdashline{1-9}[0.8pt/2pt]
ViT-Adapter-S~\citep{chen2022vision} & 36 & 288 & 15.2 & 15.3 & 50.0 & 43.1 & 54.5 & 47.1 \\
LoSA-S~\citep{mercea2024time} & 35 & 284 & 13.0 & 16.2 & 51.2 & 43.3 & 54.7 & 47.5 \\
\cellcolor[gray]{.95}\textbf{META-S$_{{\textrm{(Ours)}}}$} & \cellcolor[gray]{.95}33& \cellcolor[gray]{.95}279 & \cellcolor[gray]{.95}8.1 & \cellcolor[gray]{.95}17.6 & \cellcolor[gray]{.95}55.6 & \cellcolor[gray]{.95}44.0 & \cellcolor[gray]{.95}55.2 & \cellcolor[gray]{.95}48.5 \\
\hline \hline 
\end{tabular}
\vspace{-3mm}
\caption{Result comparisons with SOTA methods under Cascade Mask R-CNN~\citep{cai2019cascade}, ATSS~\citep{zhang2020bridging}, and GFL~\citep{li2020generalized} on the \emph{val} set of MS-COCO.}
\label{tab2}}
\vspace{-2mm}
\end{table}
% -------------------------------

% -------------------------------
\myparagraph{Superiority performance under different baselines.} 
% -------------------------------
In addition to Mask R-CNN, we also choose Cascade Mask R-CNN~\citep{cai2019cascade}, ATSS~\citep{zhang2020bridging}, and GFL~\citep{li2020generalized} as the baselines as in~\cite{chen2022vision,mercea2024time}. We explore the effectiveness and superiority performance of META on these baseline models, where the 3$\times$ training with MS strategy is used. 
The experimental results are given in Table~\ref{tab2}. We can observe that META can consistently improve performance across different baselines, and exhibits more accurate and more efficient advantages compared to SOTA ViT adapter methods. 
For example, based on the Cascade Mask R-CNN~\citep{cai2019cascade}, META-S can achieve up to $44.8\%$ AP$^\textrm{m}$ with only $83$\textbf{M} parameters, $797$\textbf{G} FLOPs, and $8.1$\textbf{GB} MC, which is $1.3\%$ AP$^\textrm{m}$ higher, $3$\textbf{M} fewer parameters, $4$\textbf{G} fewer FLOPs, and $7.1$\textbf{GB} fewer MC compared to the competitive ViT-Adapter-T method~\citep{chen2022vision}. 
Besides, META-B can achieve a $1.7\%$ AP$^\textrm{b}$/$0.5\%$ AP$^\textrm{m}$ gain on the 3$\times$ training schedule, with $5$\textbf{M} fewer parameters and $5$\textbf{G} fewer FLOPs than the ViT-Adapter-B~\citep{chen2022vision} model.
On ATSS and GFL, our method also achieves competitive $43.2\%$ and $44.$0\% AP$^\textrm{m}$, respectively, which are $0.7\%$ and $0.9\%$ AP$^\textrm{m}$ higher than the competitive vision adapter ViT-Adapter-S~\citep{chen2022vision}. 
In terms of efficiency, compare with the SOTA LoSA~\citep{mercea2024time} method,  our method has fewer parameters ($33$\textbf{M} v.s. $35$\textbf{M}), fewer FLOPs ($265$\textbf{G} v.s. $268$\textbf{G} under ATSS, $279$\textbf{G} v.s. $284$\textbf{G} under GFL), and fewer MC ($8.1$\textbf{GB} v.s. $13.0$\textbf{GB} under ATSS and GFL), indicating its higher computation and memory efficiency.
%  \input{tables/table3}

% -------------------------------
% \myparagraph{Results on different pre-trained weights.} 
% -------------------------------
% In Table~\ref{tab3}, we present the experimental results of META with different pre-trained weights and compare them with other SOTA methods including SwinViT~\citep{liu2021swin} and ViT-Adapter~\citep{chen2022vision} as in~\citep{chen2022vision}. Mask R-CNN~\citep{he2017mask} is used as the baseline, and ViT-B~\citep{li2022exploring} is used as the backbone. The 3$\times$ training schedule with MS training strategy is used. From this table, we can observe that our method is applicable to different pre-trained weights (\ie, ImageNet-1k~\citep{deng2009imagenet}, ImageNet-22k~\citep{steiner2021train}, and Multi-Modal~\citep{zhu2022uni}), and achieves more accurate AP with fewer model parameters and FLOPs compared to ViT-Adapter~\citep{chen2022vision}, across different pre-trained weights. Our method achieves $0.7\%$, $0.6\%$, and $0.6\%$ higher AP than ViT-Adapter on these three weights, respectively. These results demonstrate the superiority of our method in terms of flexibility, accuracy and efficiency.
% -------------------------------
% -------------------------------
\subsection{Experiments on Semantic Segmentation (SSeg)}
\label{sec4:3}
% -------------------------------
\myparagraph{Baselines and Settings.} 
% Experiments on SSeg are conducted using the MMSegmentation framework~\citep{mmseg2020}. 
Following~\citep{chen2022vision,jie2023fact,jie2022convolutional}, we select Semantic FPN~\citep{kirillov2019panoptic} and UperNet~\citep{xiao2018unified} as baseline models, where the Semantic FPN is trained for $80$k iterations and the UperNet is trained for $160$k iterations as in~\citep{wang2021pyramid,liu2021swin}. 
%The input images are cropped to a fix size of 512 $\times$ 512 pixels as in~\citep{xiong2024efficient,chen2022vision}. The training batch size is set to $16$, and AdamW~\citep{loshchilov2017decoupled} is used as the optimizer with the initial learning rate of $1 \times 10^{-5}$ and the weight decay of $0.05$. Following~\citep{li2022exploring,liu2021swin}, the layer-wise learning rate decay is set to $0.9$ and the drop path rate is set to $0.4$. We report the experimental results on both single scale training and MS training strategies. 
Unless otherwise specified, the training and inference settings are set up to be consistent with the ViT-Adapter~\citep{chen2022vision}.
% -------------------------------
\begin{table*}[t]
\centering
\small
\renewcommand\arraystretch{1.1}
\setlength{\tabcolsep}{.6pt}{
\begin{tabular}{l|c|cccc|cccc}
\hline \hline 
\multirow{2}{*}{Methods} & \multirow{2}{*}{Pre-Training} & \multicolumn{4}{c|}{Semantic FPN 80k} & \multicolumn{4}{c}{UperNet $160$k}  \\
~ &  ~ & \footnotesize \#P & MC & mIoU & mIoU$_{(w/\textrm{MS})}$ & \#P & MC & mIoU & mIoU$_{(w/\textrm{MS})}$ \\
\hline 
PVT-Tiny~\citep{wang2021pyramid} & IN-1k  & 17.0 & NA & 36.6 & 37.3 & 43.2 & NA & 38.5 & 39.0 \\
ViT-T~\citep{li2021benchmarking} & IN-1k  & 10.2 & NA & 39.4 & 40.5 & 34.1 & NA & 41.7 & 42.6 \\
\cdashline{1-10}[0.8pt/2pt] 
ViT-Adapter-T~\citep{chen2022vision} & IN-1k  & 12.2  & 14.0 & 41.7 & 42.1 & 36.1 & 14.0 & 42.6 & 43.6 \\
LoSA-T~\citep{mercea2024time} & IN-1k & 11.1 & 10.7 & 40.8 & 41.9 & 36.3 & 10.7 & 42.1 & 43.7 \\
\cellcolor[gray]{.95}\textbf{META-T$_{{\textrm{(Ours)}}}$} & \cellcolor[gray]{.95}IN-1k & \cellcolor[gray]{.95}10.8  & \cellcolor[gray]{.95}5.9 & \cellcolor[gray]{.95}42.2 & \cellcolor[gray]{.95}43.0 & \cellcolor[gray]{.95}35.0  & \cellcolor[gray]{.95}5.9 & \cellcolor[gray]{.95}43.1 & \cellcolor[gray]{.95}44.1 \\
\hline 
% PVT-Small~\citep{wang2021pyramid} & IN-1k  & 28.2 & 41.9 & 42.3 & 54.5 & 43.7 & 44.0 \\
% PVTv2-B2~\citep{wang2022pvt} & IN-1k  & 29.1 & 45.2 & 45.7 & -- & -- & -- \\
% Swin-T~\citep{liu2021swin} & IN-1k  & 31.9 & 41.5 & -- & 59.9 & 44.5 & 45.8 \\
Twins-SVT-S~\citep{chu2021twins} & IN-1k  & 28.3 & NA & 43.2 & 44.1 & 54.4 & NA & 46.2 & 47.1  \\
ViT-S~\citep{li2021benchmarking} & IN-1k  & 27.8 & NA & 44.6 & 45.8 & 53.6 & NA & 44.6 & 45.7 \\
\cdashline{1-10}[0.8pt/2pt] 
ViT-Adapter-S~\citep{chen2022vision} & IN-1k  & 31.9 & 14.0 & 46.1 & 46.6 & 57.6 & 14.0 & 46.2 & 47.1 \\
LoSA-S~\citep{mercea2024time} & IN-1k & 29.8 & 10.7 & 45.7 & 46.7 & 54.9 & 10.7 & 45.8  & 46.6 \\
\cellcolor[gray]{.95}\textbf{META-S$_{{\textrm{(Ours)}}}$} & \cellcolor[gray]{.95}IN-1k & \cellcolor[gray]{.95}29.4  & \cellcolor[gray]{.95}5.9 & \cellcolor[gray]{.95}46.9 & \cellcolor[gray]{.95}49.5 & \cellcolor[gray]{.95}54.7  & \cellcolor[gray]{.95}5.9 & \cellcolor[gray]{.95}47.2 & \cellcolor[gray]{.95}48.0 \\
\hline 
% Swin-B~\citep{liu2021swin} & IN-1k  & 91.2 & 46.0 & -- & 121.0 & 48.1 &49.7 \\
Twins-SVT-L~\citep{chu2021twins} & IN-1k  & 103.7 & NA & 46.7 & 45.5 & 133.0 & NA & 48.8 & 50.2\\
ViT-B~\citep{li2021benchmarking} & IN-1k  & 98.0  & NA & 46.4 & 47.6 & 127.3 & NA & 46.1 & 47.1\\
\cdashline{1-10}[0.8pt/2pt] 
ViT-Adapter-B~\citep{chen2022vision} & IN-1k  & 104.6 & 14.0 & 47.9 & 48.9 & 133.9  & 14.0 & 48.8 & 49.7 \\
LoSA-B~\citep{mercea2024time} & IN-1k & 103.5 & 10.7 & 48.3 & 49.2 & 131.2 & 10.7 & 47.3  & 48.5 \\
\cellcolor[gray]{.95}\textbf{META-B$_{{\textrm{(Ours)}}}$} & \cellcolor[gray]{.95}IN-1k & \cellcolor[gray]{.95}100.4  & \cellcolor[gray]{.95}5.9 & \cellcolor[gray]{.95}48.7 & \cellcolor[gray]{.95}49.7 & \cellcolor[gray]{.95}129.1  & \cellcolor[gray]{.95}5.9 & \cellcolor[gray]{.95}49.4 & \cellcolor[gray]{.95}50.5 \\
\hline 
Swin-L$^\ddag$~\citep{liu2021swin} & IN-22k & 234.0 & NA & --  & --  & 234.0 & NA & 52.1 & 53.5 \\
ViT-Adapter-B$^\ddag$~\citep{chen2022vision} & IN-22k  & 104.6 & 14.0 & 50.7 & 51.9 & 133.9  & 14.0 & 51.9 & 52.5 \\
ViT-Adapter-L$^\ddag$~\citep{chen2022vision} & IN-22k  & 332.0 & 14.0 & 52.9 & 53.7 & 363.8  & 14.0 & 53.4 & 54.4 \\
LoSA-L$^\ddag$~\citep{mercea2024time} & IN-22k & 310.5 & 10.7 & 52.3 & 54.0 & 338.5& 10.7 & 53.0 & 54.5\\
\cellcolor[gray]{.95}\textbf{META-L$^\ddag$$_{{\textrm{(Ours)}}}$} & \cellcolor[gray]{.95}IN-22k  & \cellcolor[gray]{.95}307.7 & \cellcolor[gray]{.95}5.9 & \cellcolor[gray]{.95}53.7 & \cellcolor[gray]{.95}54.5 & \cellcolor[gray]{.95}336.3  & \cellcolor[gray]{.95}5.9 & \cellcolor[gray]{.95}54.0 & \cellcolor[gray]{.95}55.2 \\
\hline 
ViT-Adapter-L$^*$~\citep{chen2022vision} & MM & 332.0 & 14.0 & 54.2 & 54.7 & 363.8  & 14.0 & 55.0 & 55.4 \\
LoSA-L$^*$~\citep{mercea2024time} & MM & 310.1 & 10.7 & 53.3 & 54.5 & 341.0 & 10.7 & 55.1 & 55.5 \\
\cellcolor[gray]{.95}\textbf{META-L$^*$$_{{\textrm{(Ours)}}}$} & \cellcolor[gray]{.95}MM & \cellcolor[gray]{.95}307.7  & \cellcolor[gray]{.95}5.9 & \cellcolor[gray]{.95}54.7 & \cellcolor[gray]{.95}55.2 & \cellcolor[gray]{.95}336.3  & \cellcolor[gray]{.95}5.9 & \cellcolor[gray]{.95}55.6 & \cellcolor[gray]{.95}55.9 \\
\hline \hline 
\end{tabular}
\vspace{-2mm}
\caption{Result comparisons with SOTA methods on the \emph{val} set of ADE20K. ``--'' denotes there is no such a result in its paper. ``IN'' and ``MM'' denotes ImageNet and Multi-Modal, respectively.}
\label{tab4}}
\vspace{-3mm}
\end{table*}
% -------------------------------

% -------------------------------
\myparagraph{Comparisons with state-of-the-art methods.} 
We show the SSeg experimental results of META under different settings and compare them with SOTA SSeg methods in Table~\ref{tab4}. We can observe that \emph{\textbf{\romannumeral1}}) Our method can consistently improve performance across different baselines, model scales, training strategies, and pre-training weights (including IN-1K, IN-22K, and MM), while having fewer model parameters and memory consumption compared to the advanced ViT-Adapter~\citep{chen2022vision} and LoSA~\citep{mercea2024time} methods. This demonstrates the strong applicability and learning ability of META, which can solve the problem of limited applicability in certain scenarios of downstream models in the traditional pre-training and fine-tuning strategy at the application level. \emph{\textbf{\romannumeral2}}) Compared to SOTA methods, META also can achieve a new SOTA accuracy-cost trade-off. The parameter from META is only about $0.2\%$ of the overall model parameter count, indicating a minimal impact on the total parameter count. Furthermore, the memory consumption of our method accounts for only 55.14\% of the SOTA LoSA~\citep{mercea2024time}.
\emph{\textbf{\romannumeral3}}) On larger model scales, META achieves higher efficiency with fewer parameters (\eg, -$1.4$\textbf{M} on {META-T} and -$4.2$\textbf{M} on {META-B}). This indicates that our method is suitable for fine-tuning large ViT models.
% -------------------------------
% -------------------------------
\subsection{Ablation Analysis}
\label{sec4:4}
% -------------------------------
In our ablation analysis for ISeg and ODet,
%we aim to explore the answers of the following two key questions: 1)~\emph{what is the contribution of each component in META to the overall performance?} 2)~\emph{how does the performance of our MEA compare to other modules that can be used for adapters?} To find the answers to these questions, 
we adopt Mask R-CNN~\citep{he2017mask} under the 3$\times$ training schedule with MS as the baseline and report the experimental results on the \emph{val} set of MS-COCO~\citep{caesar2018coco}, where the ImageNet-1k pre-trained ViT-B~\citep{li2022exploring} is used as the backbone. For SSeg, we choose UperNet~\citep{xiao2018unified} with 160k iterations as the baseline, where the ImageNet-1k pre-trained ViT-B~\citep{li2022exploring} is used as the backbone. We report the single-scale testing results on the \emph{val} set of ADE20K~\citep{zhou2017scene}.
Due to page limitations, we only present the results of the effectiveness of each component in the main paper. Other aspects, such as result comparisons with other variants of ViT adapter methods, will be provided in the supplementary materials as part of the ablation analysis.

\begin{table}[t]
\centering
\small
\renewcommand\arraystretch{1.1}
\setlength{\tabcolsep}{.5pt}{
\begin{tabular}{c|cccc|cccccc}
\hline \hline 
\multicolumn{11}{c}{ISeg and ODet results on the \emph{val} set of MS-COCO~\citep{caesar2018coco}} \\
\hline 
ViT-B~\citep{li2021benchmarking} & \textcolor{red}{Attn Branch.} & \textcolor{blue}{FFN Branch.} & \textcolor{orange}{Conv Branch.} & Cascade & AP$^\textrm{m}$ & AP$^\textrm{b}$ & \#P & FLOPs & MC & FPS \\
\hline 
\cmark & & & & & 41.3 & 45.8 & 113.6 & 719 & NA & 11.5 \\
\cdashline{1-11}[0.8pt/2pt] 
\cmark & \cmark & & & & 32.1 & 33.6 & 113.9 & 719 & 7.5 &  11.3 \\
\cmark & \cmark & \cmark & & & 43.4 & 46.5 & 114.4 & 719 & 7.5 & 11.3 \\
\cmark & & \cmark & & & 30.8 & 31.2 & 114.0 & 719  & 0.4 & 11.4\\
\cmark & \cmark & \cmark & \cmark & & 44.1 & 48.1 & 115.3 & 719 & 7.5 &  10.2 \\
\cmark & \cmark & \cmark & & \cmark & 43.6 & 49.3 & 114.4 & 720 & 8.1 & 11.0 \\
\cmark & \cmark & \cmark & \cmark & \cmark & 44.3 & 51.2 & 115.3 & 720 & 8.1 & 11.1 \\
\hline 
\multicolumn{11}{c}{SSeg results on the \emph{val} set of ADE20K~\citep{zhou2017scene}} \\
\hline 
UperNet~\citep{xiao2018unified} & \textcolor{red}{Attn Branch.} & \textcolor{blue}{FFN Branch.} & \textcolor{orange}{Conv Branch.} & Cascade & \multicolumn{2}{c}{mIoU} & \#P & FLOPs & MC & FPS \\
\hline 
\cmark & & & & & \multicolumn{2}{c}{46.1} & 127.3 & 1,841 & NA & 25.4 \\
\cdashline{1-11}[0.8pt/2pt] 
\cmark & \cmark & & & & \multicolumn{2}{c}{33.7} & 128.0 & 1,843 & 5.1 & 25.0 \\
\cmark & \cmark & \cmark & & & \multicolumn{2}{c}{46.8} & 128.5 & 1,846 & 5.1 & 24.9\\
\cmark & & \cmark & & & \multicolumn{2}{c}{31.1} & 128.1 & 1,841 & 0.3 & 25.0\\
\cmark & \cmark & \cmark & \cmark &  & \multicolumn{2}{c}{47.9} & 129.0 & 1,847 & 5.1 & 23.5\\
\cmark & \cmark & \cmark & & \cmark & \multicolumn{2}{c}{48.0} & 128.6 & 1,846 & 5.9 & 25.0\\
\cmark & \cmark & \cmark & \cmark & \cmark & \multicolumn{2}{c}{49.4} & 129.1 & 1,847 &  5.9 & 24.9 \\
\hline \hline 
\end{tabular}
\vspace{-2mm}
\caption{Effectiveness of each component of META. ``NA'' denotes not applicable.}
\label{tab5}}
\vspace{-4mm}
\end{table}
% -------------------------------
\myparagraph{Effectiveness of each component.} The experimental results of gradually adding META components on the pre-trained ViT-B~\citep{li2022exploring} model are shown in Table~\ref{tab5}. First, we can observe that directly using the Attn branch as the adapter on the pre-trained ViT-B~\citep{li2022exploring} model not only fails to improve performance but also significantly decreases AP$^\textrm{m}$ and AP$^\textrm{b}$ (\ie, from $41.3\%$ to $32.1\%$ and from $45.8\%$ to $33.6\%$). This indicates that using attention alone as the feature interaction method in the adapter is not enough.
Then, on this basis, when we use both the Attn branch and FFN branch as adapter components (where the layer normalization is used as a shared manner), the model's accuracy is significantly improved, surpassing the baseline model by $2.1\%$ AP$^\textrm{m}$/$0.7\%$ AP$^\textrm{b}$ and achieving $43.4\%$ AP$^\textrm{m}$/$46.5\%$ AP$^\textrm{m}$. 
At the same time, the model only has a small increase in parameters and no increase in FLOPs. The same conclusion can be drawn from the SSeg task.
\emph{Furthermore, the independent utilization of the FFN branch as the adapter component does not yield an improvement in accuracy. Instead, it significantly detracts from the model's accuracy across all three tasks evaluated.}
After that, continuing to add the Conv branch or using the cascaded mechanism can bring sustained performance improvement, \eg, $0.7\%$ and $0.2\%$ AP$^\textrm{m}$, and $1.6\%$ and $2.8\%$ AP$^\textrm{b}$, respectively. Finally, when using all components in the adapter, the model for ISeg and ODet achieves the best performance by $44.3\%$ AP$^\textrm{m}$ and $51.2\%$ AP$^\textrm{b}$ with $115.3$\textbf{M} parameters, $720$\textbf{G} FLOPs, $8.1$\textbf{GB} MC, and $11.1$ FPS. The model for SSeg achieves the best accuracy by 49.4 mIoU with $129.1$\textbf{M} parameters, $1847$\textbf{G} FLOPs, $5.9$\textbf{GB} MC, and 24.9 FPS.
In addition, results in the ablation analysis reveal that the increase in the model's FLOPs and MC is mainly attributed to the utilization of the ``cascade'' mechanism. Despite this, this mechanism can lead to significant improvements in both accuracy and speed.
% -------------------------------
% \input{tables/table6}
% -------------------------------

% -------------------------------
% \myparagraph{Superiority of META.} 
% ------------------------------- 
% META is proposed as a simple and fast ViT adapter by minimizing inefficient memory access operations. In this section, we compare META with other efficient attention methods and advanced adapter methods~\citep{marouf2024mini,xia2022vision,sung2022vl}. All attention methods are used with their default settings and the same settings as the injector and extractor in ViT-adapter~\citep{chen2022vision} to ensure the fairness of the experiments. The obtained experimental results are given in Table~\ref{tab6}. First, we can observe that compared to these methods, META achieves new state-of-the-art performance in both accuracy and efficiency. We ultimately achieve an AP of $44.3\%$ with $115.3$\textbf{M} parameters, $751$\textbf{G} FLOPs, and $17.4$ FPS. It is also worth noting that our method has a significant improvement in model speed. For example, compared to the state-of-the-art ViT-adapter~\citep{chen2022vision} based on DeformableAtt~\citep{xia2022vision}, our method reduces $50.7$\textbf{M} parameters and $7$\textbf{G} FLOPs, while bringing a performance gain of $0.6$\%AP and $3.9$FPS. Since META does not make too many changes in network architecture, these results can demonstrate the advantages of our method in terms of memory access costs.

% ---------------------------------------------------
\section{Conclusion}
We proposed a simple and fast META that aims at improving the ViT model's memory efficiency and reduce memory time consumption by minimizing layer normalization and frequent reshaping operations. The main contribution of this work was the proposal of the MEA block that shares the {layer normalization} between the self-attention and feed-forward network layers, which exist in a parallel manner. Within the proposed block, the cross-shaped self-attention was employed to reduce the requirements for frequent {reshaping operations}. Moreover, a lightweight local convolutional branch is introduced into META to enrich local inductive biases. 
Experimental results on object detection, instance segmentation, and semantic segmentation tasks demonstrated that META can achieve a new state-of-the-art accuracy-cost trade-off and at a faster inference speed. In addition to the experimental validation, we also theoretical proved that META exhibits superior generalization capability and stronger adaptability compared to current ViT adapters. In the future, we will explore the effectiveness of META on more ViT architectures, such as designing memory-efficient ViT models for a wider range of computer vision tasks. Besides, reducing the model's memory footprint is identified as a promising research direction to be explored for LLMs. 

\section*{Acknowledgment}
The authors would like to thank all anonymous reviewers for their positive comments and constructive suggestions. This work was partially supported by the Hong Kong SAR RGC General Research Fund under Grant 16208823 and ACCESS - AI Chip Center for Emerging Smart Systems, sponsored by InnoHK funding, Hong Kong SAR.

% ----------------------------------------------
\bibliography{main}
\bibliographystyle{iclr2025_conference}
% ----------------------------------------------
% ----------------------------------------------
\newpage
\section*{supplementary material}
In this supplementary material, we will provide a theoretical analysis to the proposed memory efficient Transformer adapter (META) in Section~\ref{secS1}, provide a detailed description of the experimental datasets in Section~\ref{secS2}, provide a detailed description of the experimental settings in Section~\ref{secS3},
provide more result comparisons under different pre-trained weights in Section~\ref{secS4},
provide more ablation study results in Section~\ref{secS5}, show class activation map comparisons of instance segmentation before and after adding the Conv branch in Section~\ref{secS6},qualitative visualizations of instance segmentation and semantic segmentation results in Section~\ref{secS7},  as well as the pseudo-code for when the stripe size is set to $2$ in Section~\ref{secS8}. 
% -------------------------------------------
\section{Theoretical Analysis of META}
\label{secS1}
% -------------------------------------------
{\color{red}{\emph{This supplementary is for Section~3 of the main paper.}}} In this section, we will prove that META exhibits superior generalization capability and stronger adaptability compared to existing ViT adapters. 
To achieve this goal, we will prove that the proposed memory efficient adapter (MEA) block possesses larger information entropy (IE) than the existing attention-based ViT adapters~\citep{hu2022lora,jie2023fact,chen2022vision,ma2024segment,luo2023forgery,shao2023deepfake}, which provides evidence that the MEA block has more comprehensive feature representations. Then, based on the maximum mean discrepancy (MMD) theory~\citep{cheng2021neural,arbel2019maximum,wang2021rethinking}, larger IE in the ViT adapter framework leads to superior generalization capability and stronger adaptability. The detailed theoretical analysis process is as follows:

\begin{lemma}
% ---------------------------------
In any case of mutual information, the MEA block will gain larger information entropy after fusing $\textbf{X}_{vit}$ and $\textbf{X}_{con}$.
% ---------------------------------
\end{lemma}
% ---------------------------------
\begin{proof}
As introduced in Section~3.2 of the main paper, the proposed MEA block can be viewed as an operation that integrates the ViT features (\ie, the Attn branch and the FFN branch) and the convolution features (\ie, the Conv branch). Therefore, we begin by formalizing the obtained features into the following two basic elements: the ViT features and the convolution features. To formalize the learning setting, we express the ViT features as $\textbf{X}_{vit}$ and the convolution features as $\textbf{X}_{con}$. It is evident that if $\textbf{X}_{vit}$ and $\textbf{X}_{con}$ are extracted from the same image, then $\textbf{X}_{vit}$ and $\textbf{X}_{con}$ are not independently distributed, and there exists some mutual information between them~\citep{zhang2022graph,wu2021cvt,zhang2023cae,peng2021conformer}. Therefore, the IE of the fused feature of $\textbf{X}_{vit}$ and $\textbf{X}_{con}$ within the MEA block can be expressed as:
% ---------------------------------------------------
\begin{equation}
\begin{split}
\label{eqs:1}
\textrm{H}(\textbf{X}_{vit}, \textbf{X}_{con}) = \textrm{H}(\textbf{X}_{vit}) + \textrm{H}(\textbf{X}_{con}) - \textrm{I}(\textbf{X}_{vit}; \textbf{X}_{con}),
\end{split}
\end{equation}
% ---------------------------------------------------
where $\textrm{H}(\cdot)$ is utilized to calculate the IE of the given variate, which can be formulated as:
% ---------------------------------------------------
\begin{equation}
\begin{split}
\label{eqs:2}
\textrm{H}(\textbf{X}_{vit}) = -\sum P(\textbf{x}_{vit}) log(P(\textbf{x}_{vit})),\\
\textrm{H}(\textbf{X}_{con}) = -\sum P(\textbf{x}_{con}) log(P(\textbf{x}_{con})),
\end{split}
\end{equation}
% ---------------------------------------------------
where $P(\textbf{x}_{vit})$ represents the probability of $\textbf{X}_{vit}$ taking on the value of $\textbf{x}_{vit}$. The similar definition of $P(\textbf{x}_{con})$. $\textrm{I}(\cdot;\cdot)$ in Eq.~\eqref{eqs:1} is used to compute the mutual information between $\textbf{X}_{vit}$ and $\textbf{X}_{con}$, which can be expressed as:
% ---------------------------------------------------
\begin{equation}
\begin{split}
\label{eqs:3}
\textrm{I}(\textbf{X}_{vit}; \textbf{X}_{con}) = \sum\sum \textrm{P}(\textbf{X}_{vit}, \textbf{X}_{con}) \textrm{log}(\textrm{P}(\textbf{X}_{vit}, \textbf{X}_{con}) (\textrm{P}(\textbf{X}_{vit}), \textrm{P}(\textbf{X}_{con}))),
\end{split}
\end{equation}
% ---------------------------------------------------
where $\textrm{P}(\textbf{X}_{vit}, \textbf{X}_{con})$ is their joint probability distribution. 
%\textrm{P}(\textbf{X}_{vit})$ and $\textrm{P}(\textbf{X}_{con})$ are the marginal probability distributions of $\textbf{X}_{vit}$ and $\textbf{X}_{con}$, respectively. 
Since $\textrm{I}(\textbf{X}_{vit}; \textbf{X}_{con})$ is always non-negative, $\textrm{H}(\textbf{X}_{vit}, \textbf{X}_{con})$ may still be greater than $\textrm{H}(\textbf{X}_{vit})$ or $\textrm{H}(\textbf{X}_{con})$~\citep{paninski2003estimation,gabrie2018entropy}. This suggests that the IE of the features extracted by MEA is always greater than the feature representation extracted by either of them separately.

Specifically, if $\textrm{I}(\textbf{X}_{vit}; \textbf{X}_{con})$ is small, the IE gain after fusion may still be significant, which is beneficial for improving the generalization capability and adaptability of the block. However, when $\textrm{I}(\textbf{X}_{vit}; \textbf{X}_{con})$ is large, the IE gain after fusion may be reduced. This means that $\textrm{I}(\textbf{X}_{vit}; \textbf{X}_{con})$ may affect the IE improvement of the fused model. Next, we will discuss the impact of $\textbf{X}_{vit}$ and $\textbf{X}_{con}$ on improving the IE of the adapter based on the size of $\textrm{I}(\textbf{X}_{vit}; \textbf{X}_{con})$, which can be divided into the following three cases:

\begin{itemize}
% --------------------------
\item {{Small} $\textrm{I}(\textbf{X}_{vit}; \textbf{X}_{con})$.} This is an ideal state. When the dependency between $\textbf{X}_{vit}$ and $\textbf{X}_{con}$ is small, it indicates that $\textrm{I}(\textbf{X}_{vit}; \textbf{X}_{con})$ is small, that is, $\textbf{X}_{vit}$ and $\textbf{X}_{con}$ respectively represent different information of the image. In this case, fusing $\textbf{X}_{vit}$ and $\textbf{X}_{con}$ can bring a significant increase in IE, which is beneficial to improving the adapter's generalization capability and adaptability.
% --------------------------
\item {{Medium} $\textrm{I}(\textbf{X}_{vit}; \textbf{X}_{con})$.} When $\textrm{I}(\textbf{X}_{vit}; \textbf{X}_{con})$ is between small and large, it indicates that there is a certain degree of correlation between them. In this case, fusing $\textbf{X}_{vit}$ and $\textbf{X}_{con}$ may still bring some IE gain. The specific improvement effect depends on the degree of correlation between $\textbf{X}_{vit}$ and $\textbf{X}_{con}$ and their complementarity in image representations. Fortunately~\citep{zhang2022graph,zhang2023cae,marouf2024mini,liu2023efficientvit}, a large amount of work has validated that ViT and convolutional layers can extract distinctive information from images. Therefore, in this case, fusing $\textbf{X}_{vit}$ and $\textbf{X}_{con}$ can still bring IE gains.
\item {{Large} \myparagraph{$\textrm{I}(\textbf{X}_{vit}; \textbf{X}_{con})$}.} When $\textrm{I}(\textbf{X}_{vit}; \textbf{X}_{con})$ between $\textbf{X}_{vit}$ and $\textbf{X}_{con}$ is large, it indicates that there is a high correlation between them, \ie, global ViT and local convolution features may represent similar or overlapping information of the image. In this case, the IE gain brought by fusing $\textbf{X}_{vit}$ and $\textbf{X}_{con}$ may decrease because there is a lot of information overlap between them. However, in our case, the probability of such a scenario occurring is almost non-existent, fusing $\textbf{X}_{vit}$ and $\textbf{X}_{con}$ may still improve the performance of the model to some extent, because they may capture the detailed information of the image to varying degrees.
% --------------------------
\end{itemize}
% --------------------------

Based on the aforementioned theoretical analysis, we can conclude that the proposed MEA block has a larger IE than existing ViT adapters (which are primarily based on the attention mechanism) under any scenario. This provides evidence that the MEA block has more comprehensive feature representations. 
% ---------------------------------
\end{proof}
% ---------------------------------
As the MEA block includes a parallel convolutional branch, it can better capture local inductive biases compared to the traditional ViT adapter, which mainly uses self-attention~\citep{hu2022lora,jie2023fact,chen2022vision,ma2024segment,luo2023forgery,shao2023deepfake,mercea2024time}. 
Therefore, the MEA block's feature space should be more capable of distinguishing different samples, resulting in a larger MMD value. 
Our MEA block's feature space is obtained by combining the attention branch, the feed-forward network branch, and the local convolutional branch, enabling it to capture both local and global inductive biases of the given image. 
In contrast, the traditional ViT adapter's feature space is mainly obtained through self-attention and may not be able to capture local features well. Therefore, according to the MMD theory~\citep{cheng2021neural,arbel2019maximum,wang2021rethinking}, we can conclude that if the MEA block's feature space is more discriminative than the traditional ViT adapter's feature space, then the MEA block's feature space is more suitable for adapter feature space and can better improve the model's generalization capability and adaptability.

% -------------------------------------------
\section{Introduction of the Experimental Datasets}
\label{secS2}
% -------------------------------------------
{\color{red}{\emph{This supplementary is for Section~4.1 of the main paper.}}}
In our paper, two representative datasets are used to evaluate the effectiveness and efficiency of our method, including MS-COCO~\citep{caesar2018coco} for ODet and ISeg, and ADE20K~\citep{zhou2017scene} for SSeg. Below are the details of the used datasets:

% -------------------------------
\begin{itemize}
% -------------------------------
\item MS-COCO~\citep{caesar2018coco} is a representative yet challenging dataset for common scene IS and object detection, which consists of $118$k, $5$k and $20$k images for the \emph{training} set, the \emph{val} set and the \emph{test} set, respectively. In our experiments, the model is trained on the \emph{training} set and evaluated on the \emph{val} set.
% -------------------------------
\item ADE20K~\citep{zhou2017scene} is a scene parsing dataset with $20$k images and $150$ object categories. Each image has pixel-level annotations for SS of objects and regions within the scene. The dataset is divided into $20$k, $2$k, and $3$k images for \emph{training}, \emph{val} and \emph{test}, respectively. Our model is trained on the \emph{training} set and evaluated on the \emph{val} set.
% -------------------------------
\end{itemize}
% -------------------------------
For data augmentation, random horizontal flip, brightness jittering and random scaling within the range of $[0.5, 2]$ are used in training as in~\citep{chen2022vision,luo2023forgery,zhang2023cae,mercea2024time}. By default, the inference results are obtained at a single scale, unless explicitly specified otherwise.

% -------------------------------------------
\section{Introduction of the Experimental Settings}
\label{secS3}
% -------------------------------------------
{\color{red}{\emph{This supplementary is for Section~4.2 of the main paper.}}} Experiments on object detection and instance segmentation are conducted using the open-source MMDetection framework~\citep{chen2019mmdetection}. The training batch size is set to $16$, and AdamW~\citep{loshchilov2017decoupled} is used as the optimizer with the initial learning rate of $1 \times 10^{-4}$ and the weight decay of $0.05$. The layer-wise learning rate decay is used and set to $0.9$, and the drop path rate is set to $0.4$. Following~\citep{xiong2024efficient,wang2021pyramid,chen2022vision,liu2022convnet}, to ensure a fair result comparison, we choose two training schedules, 1$\times$ (\ie, $12$ training epochs) and 3$\times$ (\ie, $36$ training epochs). For the 1$\times$ training schedule, images are resized to the shorter side of 800 pixels, with the longer side not exceeding $1,333$ pixels. In inference, the shorter side of images is consistently set to 800 pixels by default. For the 3$\times$ training schedule, the multi-scale training strategy is also used as in~\citep{chen2022vision}, and the shorter side is resized to $480$ to $800$ pixels, while the longer side remains capped at $1,333$ pixels.

{\color{red}{\emph{This supplementary is for Section~4.3 of the main paper.}}} Experiments on semantic segmentation are conducted using the MMSegmentation framework~\citep{mmseg2020}. The input images are cropped to a fix size of 512 $\times$ 512 pixels as in~\citep{xiong2024efficient,chen2022vision}. The training batch size is set to $16$, and AdamW~\citep{loshchilov2017decoupled} is used as the optimizer with the initial learning rate of $1 \times 10^{-5}$ and the weight decay of $0.05$. Following~\citep{li2022exploring,liu2021swin}, the layer-wise learning rate decay is set to $0.9$ and the drop path rate is set to $0.4$. We report the experimental results on both single scale training and multi-scale training strategies. 
% -------------------------------
\begin{table}[t]
\centering
\small
\renewcommand\arraystretch{1.2}
\setlength{\tabcolsep}{6pt}{
\begin{tabular}{r|r|ccl}
\hline \hline 
Methods & Pre-Trained & Params.$\downarrow$ & AP$^\textrm{m}$ $\uparrow$ \\
\hline 
Swin-B~\citep{liu2021swin} & ImageNet-1k~\citep{deng2009imagenet} & 107.1 &  43.3 \\
ViT-Adapter-B~\citep{chen2022vision} & ImageNet-1k~\citep{deng2009imagenet} & 120.2 & 43.6 \\
\cellcolor[gray]{.95}\textbf{META-B$_{{\textrm{(Ours)}}}$} & \cellcolor[gray]{.95}ImageNet-1k~\citep{deng2009imagenet} & \cellcolor[gray]{.95}115.3 & \cellcolor[gray]{.95}44.3$_{\color{red}{+0.7}}$ \\
\cdashline{1-4}[0.8pt/2pt]
Swin-B~\citep{liu2021swin} & ImageNet-22k~\citep{steiner2021train} & 107.1 & 44.3\\
ViT-Adapter-B~\citep{chen2022vision} & ImageNet-22k~\citep{steiner2021train} & 120.2 & 44.6 \\
\cellcolor[gray]{.95}\textbf{META-B$_{{\textrm{(Ours)}}}$} & \cellcolor[gray]{.95}ImageNet-22k~\citep{steiner2021train} & \cellcolor[gray]{.95}115.3  & \cellcolor[gray]{.95}45.2$_{\color{red}{+0.6}}$ \\
\cdashline{1-4}[0.8pt/2pt]
Swin-B~\citep{liu2021swin} & Multi-Modal~\citep{zhu2022uni} & 107.1 &   -- \\
ViT-Adapter-B~\citep{chen2022vision} & Multi-Modal~\citep{zhu2022uni} & 120.2  & 45.3 \\
\cellcolor[gray]{.95}\textbf{META-B$_{{\textrm{(Ours)}}}$} & \cellcolor[gray]{.95}Multi-Modal~\citep{zhu2022uni} & \cellcolor[gray]{.95}115.3  & \cellcolor[gray]{.95}45.9$_{\color{red}{+0.6}}$ \\
\hline \hline 
\end{tabular}
\caption{Result comparisons on Params. (\textbf{M}) and AP (\%) under different pre-trained weights with Mask R-CNN ($3 \times$ +MS schedule)~\citep{he2017mask} as the baseline model on the \emph{val} set of MS-COCO~\citep{caesar2018coco}. ``--'' denotes there is no such a result in its paper.}
\label{tab3}}
\end{table}
% -------------------------------

% -------------------------------------------
\section{Result Comparisons under Different Weights}
\label{secS4}
% -------------------------------------------
{\color{red}{\emph{This supplementary is for Section~4.2 of the main paper.}}} In this section, we present the experimental results of META on object detection and instance segmentation with different pre-trained weights and compare them with other state-of-the-art methods including SwinViT~\citep{liu2021swin} and ViT-Adapter~\citep{chen2022vision} as in~\citep{chen2022vision}. 
Mask R-CNN~\citep{he2017mask} is used as the baseline, and ViT-B~\citep{li2022exploring} is used as the backbone. The 3$\times$ training schedule with MS training strategy is used. The obtained experimental results are given in Table~\ref{tab3}.
From this table, we can observe that our method is applicable to different pre-trained weights (\ie, ImageNet-1k~\citep{deng2009imagenet}, ImageNet-22k~\citep{steiner2021train}, and Multi-Modal~\citep{zhu2022uni}), and achieves more accurate AP with fewer model parameters compared to ViT-Adapter~\citep{chen2022vision}, across different pre-trained weights.  

% -------------------------------------------
\section{More Ablation Study Results}
\label{secS5}
% -------------------------------------------
{\color{red}{\emph{This supplementary is for Section~4.4 of the main paper.}}} In our main paper, we present the experimental results of deploying adapters with Attn branch and FFN branch as components on ViT-B~\citep{li2022exploring}. It is noteworthy that the layer normalization operation has been shared between the Attn branch and the FFN branch to reduce the memory access costs associated with the normalization operations. In this section, we demonstrate a result comparison between the experimental results of using shared layer normalization operation and those of not using it in the traditional setting (\ie, the non-shared normalization). The obtained experimental results are shown in Table~\ref{tab:s1}. It can be observed that sharing layer normalization does not significantly improve the performance in terms of AP. However, compared to FPS, FLOPs, MC, our approach can achieve satisfactory performance gains.
% --------------------------
\begin{table*}[t]
\centering
\renewcommand\arraystretch{1.2}
\setlength{\tabcolsep}{1pt}{
\begin{tabular}{r|ccccc|ccccc}
\hline \hline 
Settings & ViT-B & Attn & FFN & Conv & Cascade & AP$^\textrm{m}$ $\uparrow$ & FPS$\uparrow$ & Params.$\downarrow$ & FLOPs$\downarrow$ & MC$\downarrow$ \\
\hline 
Baseline model & \cmark & \xmark & \xmark & \xmark & \xmark & 41.3 & 11.5 & 113.6\textbf{M} & 719\textbf{G} & NA\\
\cdashline{1-11}[0.8pt/2pt]
\cellcolor[gray]{.95}Shared normalization & \cmark & \cmark & \cmark & \xmark & \xmark & \cellcolor[gray]{.95}43.4 & \cellcolor[gray]{.95}11.3 & \cellcolor[gray]{.95}114.4\textbf{M} & \cellcolor[gray]{.95}719\textbf{G} & \cellcolor[gray]{.95}7.5\textbf{GB}\\
Non-shared normalization & \cmark & \cmark & \cmark & \xmark & \xmark & 43.2 & 10.5 & 114.4\textbf{M} & 737\textbf{G} & 8.8\textbf{GB}\\
\hline \hline 
\end{tabular}
\caption{Ablation study results on shared layer normalization.}
\label{tab:s1}}
\end{table*}
% --------------------------

% -------------------------------
{\color{red}{\emph{This supplementary is for Section~4.4 of the main paper.}}} META is proposed as a simple and fast ViT adapter by minimizing inefficient memory access operations. In this section, we compare META with other efficient attention methods and advanced adapter methods~\citep{marouf2024mini,xia2022vision,sung2022vl}. All methods are used with their default settings and the same settings as the injector and extractor in ViT-adapter~\citep{chen2022vision}. Following the same setup as in~\citep{chen2022vision}, the attention mechanism is utilized as the ViT-adapter layer. Therefore, during the experimental comparisons, we replace the attention mechanism in the ViT-adapter with alternative attention mechanisms to ensure a fair comparison. 
The obtained experimental results are given in Table~\ref{tab6}. We can observe that compared to these methods, META achieves new state-of-the-art performance in both accuracy and efficiency. We ultimately achieve an AP of $44.3\%$ with $115.3$\textbf{M} parameters, $720$\textbf{G} FLOPs, $17.4$ FPS, and 8.1 \textbf{GB} MC. 
% -------------------------------
\begin{table}[t]
\centering
\footnotesize
\renewcommand\arraystretch{1.2}
\setlength{\tabcolsep}{5pt}{
\begin{tabular}{r|ccccc}
\hline \hline 
Methods & AP$\uparrow$ & FPS$\uparrow$ & Params. (\textbf{M})$\downarrow$ & FLOPs (\textbf{G})$\downarrow$  & Momory (\textbf{GB})$\downarrow$ \\
\hline 
WindowAtt~\citep{liu2021swin} & 41.2 & 11.6 & 145.0 & 982 & 18.5 \\
PaleAttention~\citep{wu2022pale} & 42.8 & 14.4 & 155.2 & 1,029 & 16.7\\
Attention~\citep{vaswani2017attention} & 43.1 & 5.2 & 188.4 & 1,250 & 18.3 \\
CSWindow~\citep{dong2022cswin}& 43.1 & 13.7 & 144.6 & 990 & 12.9\\
SimplingAtte~\citep{he2023simplifying} & 43.3 & 12.2 & 126.3 & 994 & 17.1\\
DeformableAtt~\citep{xia2022vision} & 43.7 & 13.5 & 166.0 & 988 & 15.2 \\
\cdashline{1-6}[0.8pt/2pt]
MiniAdapters~\citep{marouf2024mini} & 41.9 & 15.0 & 131.8 & 995 & 12.2 \\
VL-Adapter~\citep{sung2022vl} & 42.7 & 14.5 & 167.2 & 993  & 14.0\\
\cellcolor[gray]{.95}\textbf{META-B$_{{\textrm{(Ours)}}}$} & \cellcolor[gray]{.95}44.3 & \cellcolor[gray]{.95}17.4 & \cellcolor[gray]{.95}115.3 & \cellcolor[gray]{.95}720 & \cellcolor[gray]{.95}8.1\\
\hline \hline 
\end{tabular}
\caption{Result comparisons with different adapters.}
\label{tab6}}
\end{table}
% -------------------------------

% -------------------------------------------
\section{Visualizations under the Conv branch}
\label{secS6}
% -------------------------------------------
{\color{red}{\emph{This supplementary is for Section~3.2 of the main paper.}}} In this section, to observe if the adapter has learned local inductive biases through the Conv branch, we visualize the model's class activation maps. The obtained visualizations are given in Figure~\ref{figs1}. From this figure, it can be observed that after adding the Conv branch, the model focuses more on the specific object area (\eg,`` the dog'' and ``the person'') rather than the surrounding area that may extend beyond the object itself, as was the case before adding the Conv branch. This indicates that our method effectively learns local inductive biases after incorporating the Conv branch.
% -------------------------------------------
\begin{figure}[t]
\centering
\includegraphics[width=.6\textwidth]{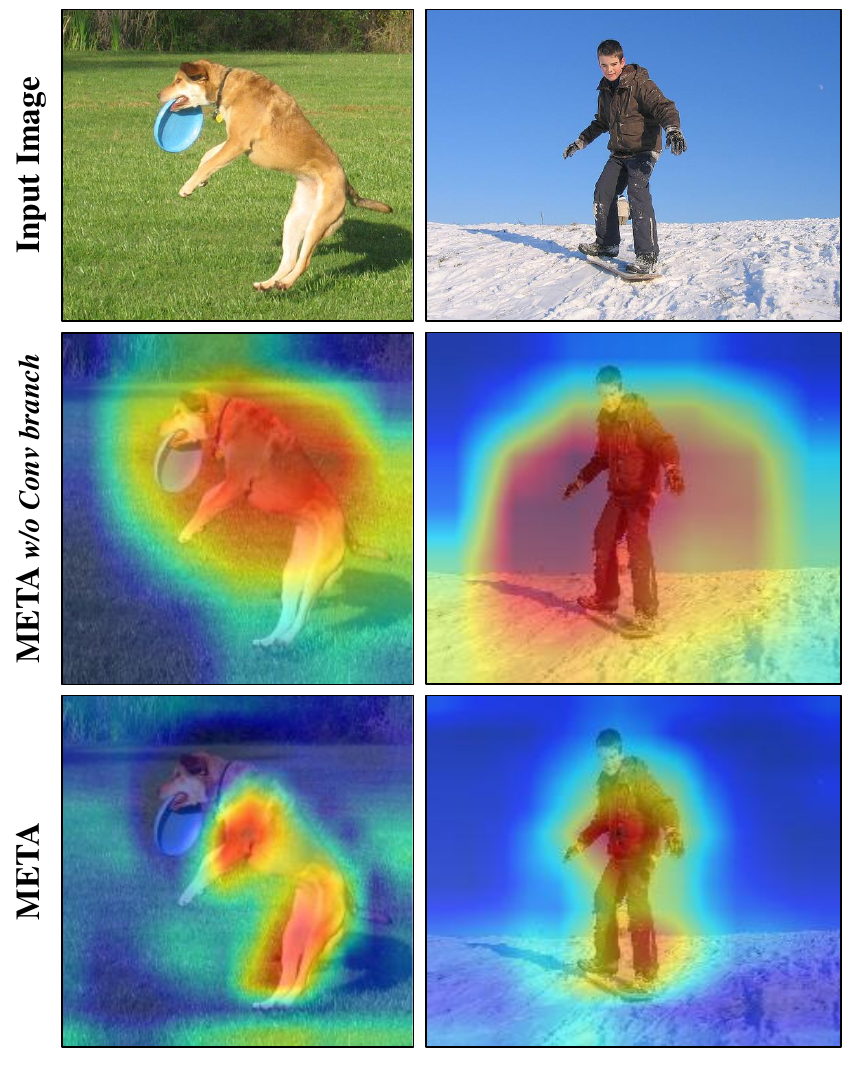}
\vspace{-4mm}
\caption{Class activation map comparisons of instance segmentation before and after adding the Conv branch. The sample images are from the \emph{training} set of MS-COCO~\citep{caesar2018coco}.}
\vspace{-4mm}
\label{figs1}
\end{figure}
% -------------------------------------------

% -------------------------------------------
\section{Qualitative Visualization results}
\label{secS7}
% -------------------------------------------
{\color{red}{\emph{This supplementary is for Section~4.2 and~4.3 of the main paper.}}} In this section, we show qualitative results on both instance segmentation and semantic segmentation. To demonstrate the superiority of our method, we present visualization results of ablation studies on instance segmentation, as well as comparisons with state-of-the-art methods on both instance segmentation and semantic segmentation. 
The obtained visualization results are shown in Figure~\ref{figs2}. From the results, it can be observed that compared to other methods, our method can achieve more accurate object masks that better fit the actual boundaries of the objects themselves.
% -------------------------------------------
\begin{figure*}[t]
\centering
\includegraphics[width=1\textwidth]{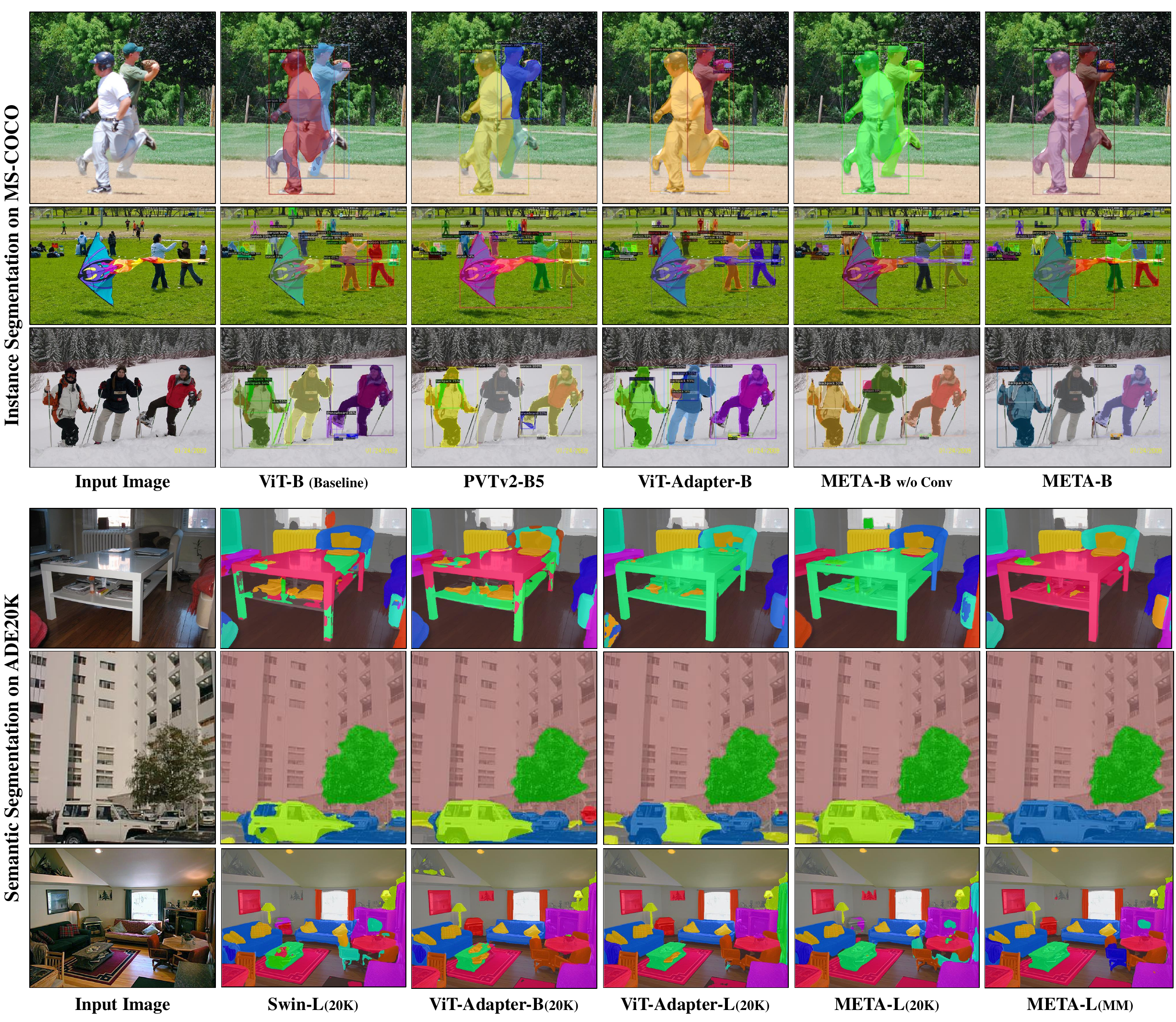}
\vspace{-6mm}
\caption{Qualitative results. The sample images are from the \emph{val} set of MS-COCO~\citep{caesar2018coco} for instance segmentation, and are from the \emph{val} set of ADE20K~\citep{zhou2017scene} for semantic segmentation.
``w/o Conv'' denotes that the Conv branch is not used in the experiments.
``20K'' and ``MM'' refers to the backbone network being pre-trained on ImageNet-22k~\citep{steiner2021train} and Multi-Modal~\citep{zhu2022uni}, respectively.}
\label{figs2}
\end{figure*}
% -------------------------------------------

as well as the pseudo-code for when the stripe size is set to $2$ in Section~\ref{secS8}. 
% -------------------------------------------
\section{Pseudo-code fo stripe size = $2$}
\label{secS8}
% -------------------------------------------
In this code snippet, stripe size is set to 2, and relevant features are directly obtained using the gather function instead of reshaping them with img2windows. This operation can reduce unnecessary reshaping operations and improves the efficiency of the code.
% -------------------------------------------
\begin{python}
function cross_shaped_window_attention(x, num_heads, window_size):
    # x: given feature
    # num_heads: head number
    # window_size: window size

    # Get dimensions
    (batch_size, seq_length, d_model) = shape(x)

    # Split into multiple heads
    Q, K, V = split_heads(x, num_heads)

    # Initialize attention output
    attention_output = zeros(batch_size, seq_length, d_model)

    # Initialize previous head's output for cascaded attention
    previous_Q = zeros(batch_size, seq_length, d_model)
    previous_K = zeros(batch_size, seq_length, d_model)
    previous_V = zeros(batch_size, seq_length, d_model)

    # Calculate attention for each head
    for head in range(num_heads):
        for position in range(seq_length):
            # Get cross-shaped window indices
            window_indices = get_cross_shaped_window_indices(position, window_size)

            # Gather Q, K, V for the current window
            Q_window = gather(Q[head], window_indices)
            K_window = gather(K[head], window_indices)
            V_window = gather(V[head], window_indices)

            # Incorporate previous head's output for cascaded attention
            if head > 0:
                Q_window += previous_Q
                K_window += previous_K
                V_window += previous_V

            # Calculate attention scores
            attention_scores = softmax(Q_window * K_window^T / sqrt(d_k))

            # Compute the attention output for the current position
            attention_output[position] = attention_scores * V_window

        # Update previous head's output for the next head
        previous_Q = Q_window
        previous_K = K_window
        previous_V = V_window

    # Final linear transformation
    attention_output = linear_transform(attention_output)
    return attention_output

function feed_forward_network(x):
    # Feed Forward Network
    x = ReLU(linear(x))
    x = linear(x)
    return x
\end{python}

\begin{python}
def get_cross_shaped_window_indices(position, window_size, seq_length):
    # Initialize the list of indices
    indices = []

    # Add the current position
    indices.append(position)

    # Add vertical neighbors (up and down)
    for offset in range(-window_size, window_size + 1):
        if position + offset >= 0 and position + offset < seq_length:
            indices.append(position + offset)

    # Remove duplicates and sort the indices
    indices = list(set(indices))
    indices.sort()

    return indices
\end{python}
% ----------------------------------------------

% ----------------------------------------------
\end{document}